\documentclass[letterpaper]{article}
\usepackage{uai2019}
\usepackage[margin=1in]{geometry}
\usepackage{times}
\usepackage[utf8]{inputenc}
\usepackage[english]{babel}
\usepackage{graphicx}
\usepackage{tikz}
\usetikzlibrary{automata,positioning}
\usepackage{verbatim}
\usepackage{enumitem}
\usepackage{amsfonts}
\usepackage{algorithm}
\usepackage{float}
\usepackage{units}
\usepackage{subcaption}
\usepackage[noend]{algpseudocode}
\usepackage{booktabs}
\usepackage{xr}

\newcommand\ci{\perp\!\!\!\perp}
\newcommand\bias[1]{{\bf 1}_{ #1 }}

\newcommand{\G}{{\mathcal G}}

\newcommand{\quotient}[2]{\ensuremath{{}^{#1}\!/\!_{#2}}}

\usepackage{amsthm}
\newtheorem{cor}{Corollary}
\newtheorem{prop}{Proposition}
\newtheorem{lemma}{Lemma}

\theoremstyle{remark}

\theoremstyle{definition}

\usepackage{amsmath}
\usepackage{amssymb}
\usepackage[hidelinks]{hyperref}

\DeclareMathOperator{\pa}{pa}
\DeclareMathOperator{\ch}{ch}
\DeclareMathOperator{\de}{de}
\DeclareMathOperator{\nd}{nd}
\DeclareMathOperator{\an}{an}

\DeclareMathOperator{\dis}{dis}
\DeclareMathOperator{\mb}{mb}

\newenvironment{lema}[1]{\par\noindent{\bf Lemma #1\ }\em}{\em}

\newenvironment{proa}[1]{\par\noindent{\textbf{Proposition #1} }\em}{\em}

\newcommand{\red}{\textcolor{red}}

\usepackage{enumitem}

\usepackage{multirow}
\usepackage{pifont}
\usepackage{lipsum}
\usepackage{arydshln}

\title{Identification In Missing Data Models Represented By \\ Directed Acyclic Graphs
}

\author{Rohit Bhattacharya$^{\dagger *}$, Razieh Nabi$^{\dagger *}$, Ilya Shpitser$^{\dagger}$, James M. Robins$^{\ddagger}$ \\ 
{$^{\dagger}$ Department of Computer Science, Johns Hopkins University, Baltimore, MD} \\ {$^{\ddagger}$ Department of Epidemiology, Harvard T. H. Chan School of Public Health, Boston, MA} \\ 
{$^{*}$ Equal contribution} \\
(\textit{rbhattacharya}@, \textit{rnabi}@, \textit{ilyas}@cs.)jhu.edu, \textit{robins}@hsph.harvard.edu }

\begin{document}
	
\maketitle

\begin{abstract}
Missing data is a pervasive problem in data analyses, resulting in datasets that contain censored realizations of a target distribution. Many approaches to inference on the target distribution using censored observed data, rely on missing data models represented as a factorization with respect to a directed acyclic graph. In this paper we consider the identifiability of the target distribution within this class of models, and show that the most general identification strategies proposed so far retain a significant gap in that they fail to identify a wide class of identifiable distributions. To address this gap, we propose a new algorithm that significantly generalizes the types of manipulations used in the ID algorithm \cite{shpitser06id, tian02general}, developed in the context of causal inference, in order to obtain identification. 

%
%
\end{abstract}

\section{INTRODUCTION}
Missing data is ubiquitous in applied data analyses resulting in target distributions that are systematically censored by a missingness process. A common modeling approach assumes data entries are censored in a way that does not depend on the underlying missing data, known as the missing completely at random (MCAR) model, or only depends on observed values in the data, known as the missing at random (MAR) model. These simple models are insufficient however, in problems where 
missingness status may depend on underlying values that are themselves censored.  This type of missingness is known as missing not at random (MNAR) \cite{robins97non, rubin76inference, tsiatis06missing}.

While the underlying target distribution is often not identified from observed data under MNAR, there exist identified MNAR models.  These include the permutation model \cite{robins97non}, the discrete choice model \cite{tchetgen16discrete}, the no self-censoring model \cite{sadinle16itemwise, shpitser2016consistent}, the block-sequential MAR model \cite{zhou2010block}, and others.  Restrictions defining many, but not all, of these models may be represented by a factorization of the full data law (consisting of both the target distribution and the missingness process) with respect to a directed acyclic graph (DAG).

The problem of identification of the target distribution from the observed distribution in missing data DAG models bears many similarities to the problem of identification of interventional distributions from the observed distribution in causal DAG models with hidden variables.  This observation prompted recent work \cite{mohan14missing, mohan2013missing, shpitser15missing} on adapting identification methods from causal inference to identifying target distributions in missing data models.

In this paper we show that the most general currently known methods for identification in missing data DAG models retain a significant gap, in the sense that they fail to identify the target distribution in many models where it is identified.  We show that methods used to obtain a complete characterization of identification of interventional distributions, via the ID algorithm \cite{shpitser06id, tian02general}, or their simple generalizations \cite{mohan14missing, mohan2013missing, shpitser15missing}, are insufficient on their own for obtaining a similar characterization for missing data problems.  We describe, via a set of examples, that in order to be complete, an identification algorithm for missing data must recursively simplify the problem by removing \emph{sets} of variables, rather than single variables, and these must be removed according to a \emph{partial order}, rather than a total order.  Furthermore, the algorithm must be able to handle subproblems where selection bias or hidden variables, or both, are present even if these complications are missing in the original problem.  We develop a new general algorithm that exploits these observations and significantly narrows the identifiability gap in existing methods. Finally, we show that in certain classes of missing data DAG models, our algorithm takes on a particularly simple formulation to identify the target distribution.



Our paper is organized as follows.  In section \ref{sec:intro}, we introduce the necessary preliminaries from the graphical causal inference literature.  In section \ref{sec:missing} we introduce missing data models represented by DAGs.  In section \ref{sec:examples}, we illustrate, via examples, that existing identification strategies based on simple generalizations of causal inference methods are not sufficient for identification in general, and describe generalizations needed for identification in these examples.  In section \ref{sec:alg}, we give a general identification algorithm which incorporates techniques needed to obtain identification in the examples we describe.  
Section \ref{sec:conclusion} contains our conclusions.  We defer longer proofs to the supplement in the interests of space.

\section{
PRELIMINARIES}
\label{sec:intro}

Many techniques useful for identification in missing data contexts were first derived in causal inference.  Causal inference is concerned with expressing counterfactual distributions, obtained after the intervention operation, from the observed data distribution, using constraints embedded in a causal model, often represented by a DAG.

A DAG is a graph $\cal{G}$ with a vertex set $\bf{V}$ connected by directed edges such that there are no directed cycles in the graph.
A statistical model of a DAG ${\cal G}$ is the set of distributions $p({\bf V})$ such that $p({\bf V}) = \prod_{V \in {\bf V}} p(V | \pa_{\cal G}(V))$,
where $\pa_{\cal G}(V)$ are the set of parents of $V$ in ${\cal G}$.
Causal models of a DAG are also sets of distributions, but on counterfactual random variables.  Given $Y \in {\bf V}$ and ${\bf A} \subseteq {\bf V} \setminus \{ Y \}$, a counterfactual variable, or potential outcome, written as $Y({\bf a})$, represents the value of $Y$ in a hypothetical situation where ${\bf A}$ were set to values ${\bf a}$ by an \emph{intervention operation} \cite{pearl09causality}.
Given a set ${\bf Y}$, define ${\bf Y}({\bf a}) \equiv \{ {\bf Y} \}({\bf a}) \equiv \{ Y({\bf a}) \mid Y \in {\bf Y} \}$.  The distribution $p({\bf Y}({\bf a}))$ is sometimes written as $p({\bf Y} | \text{do}({\bf a}))$ \cite{pearl09causality}.



A causal parameter is said to be \emph{identified} in a causal model if it is a function of the observed data distribution $p({\bf V})$.
Otherwise the parameter is said to be \emph{non-identified}.
In all causal models of a DAG ${\cal G}$ that are typically used, all interventional distributions 
$p(\{ {\bf V} \setminus {\bf A} \}({\bf a}))$
are identified by the \emph{g-formula} \cite{robins86new}:
{\small
\begin{align}
p(\{ {\bf V} \setminus {\bf A} \}({\bf a})) =
\!\!\!
\prod_{V \in {\bf V} \setminus {\bf A}}
\left.
\!\!\!
p(V | \pa_{\cal G}(V)) \right|_{{\bf A}={\bf a}}.
\label{eqn:g}
\end{align}
}%
If a causal model contains hidden variables, only data on the observed marginal distribution is available.  In this case, not every interventional distribution is identified, and identification theory becomes more complex.  A general algorithm for identification of causal effects in this setting was given in \cite{tian02general}, and proven complete in \cite{shpitser06id,huang06do}.  Here, we describe a simple reformulation of this algorithm as a truncated nested factorization analogous to the g-formula, phrased in terms of kernels and mixed graphs recursively defined via a fixing operator \cite{richardson17nested}.  As we will see, many of the techniques developed for identification in the presence of hidden variables will need to be employed (and generalized) for missing data, even if no variables are completely hidden.


We describe acyclic directed mixed graphs (ADMGs) obtained from a hidden variable DAG by a latent projection operation in section \ref{subsec:latent}, and a nested factorization associated with these ADMGs 
in section \ref{subsec:nested}.  This factorization is formulated in terms of conditional ADMGs and kernels (described in section \ref{subsubsec:cond}), via the fixing operator (described in section \ref{subsubsec:fix}).  The truncated nested factorization that yields all identifiable functions for interventional distributions is described in section \ref{subsec:id}.

As a prelude to the rest of the paper, we introduce the following notation for some standard genealogic sets of a graph $\mathcal{G}$ with a set of vertices ${\bf V}$: parents $\pa_{\mathcal{G}}(V) \equiv \{U \in {\bf V} | U \rightarrow V \}$, children $\ch_{\mathcal{G}}(V) \equiv \{U \in {\bf V} | V \rightarrow U \}$, 
descendants $\de_{\mathcal{G}}(V) \equiv \{U \in {\bf V} | V \rightarrow \cdots \rightarrow U \}$, ancestors $\an_{\mathcal{G}}(V) \equiv \{U \in {\bf V} | U \rightarrow \cdots \rightarrow V \}$, and non-descendants $\nd_{\mathcal{G}}(V) \equiv {\bf V}\setminus \de_{\mathcal{G}}(V)$. A district ${\bf D}$ is defined as the maximal set of vertices that are pairwise connected by a bidirected path (a path containing only $\leftrightarrow$ edges). We denote the district of $V$ as $\dis_{\mathcal{G}}(V)$, and the set of all districts in $\mathcal{G}$ as $\mathcal{D}(\mathcal{G})$. By convention, for any $V$, $\dis_{\mathcal{G}}(V) \cap \de(V) \cap \an_{\mathcal{G}}(V) = \{V\}$. Finally, the Markov blanket $\mb_{\cal G}(V)\equiv \dis_{\mathcal{G}}(V) \cup \pa_{\cal G}(\dis_{\mathcal{G}}(V))$ is defined as the set that gives rise to the following independence relation through m-separation: $V \ci \nd_{\cal G}(V) \setminus \mb_{\cal G}(V) | \mb_{\cal G}(V)$ \cite{richardson17nested}. The above definitions apply disjunctively to sets of variables ${\bf S} \subset {\bf V}$; e.g. $\pa_{\mathcal{G}}({\bf S}) = \cup_{S \in {\bf S}} \pa_{\mathcal{G}}(S)$.
%
%
\subsection{LATENT PROJECTION ADMGS}
\label{subsec:latent}

Given a DAG ${\cal G}({\bf V}\cup{\bf H})$, where ${\bf V}$ are observed and ${\bf H}$ are hidden variables, a {latent projection} ${\cal G}({\bf V})$ is the following ADMG with a vertex set ${\bf V}$. An edge $A \to B$ exists in ${\cal G}({\bf V})$ if there exists a directed path from $A$ to $B$ in ${\cal G}({\bf V}\cup{\bf H})$ with all intermediate vertices in ${\bf H}$.  Similarly, an edge $A \leftrightarrow B$ exists in ${\cal G}({\bf V})$ if there exists a path without consecutive edges $\to \circ \gets$ from $A$ to $B$ with the first edge on the path of the form $A \gets$ and the last edge on the path of the form $\to B$, and all intermediate vertices on the path in ${\bf H}$.  Latent projections define an infinite class of hidden variable DAGs that share identification theory.  Thus, identification algorithms are typically defined on latent projections for simplicity.

\subsection{NESTED FACTORIZATION}
\label{subsec:nested}

The nested factorization of $p({\bf V})$ with respect to an ADMG ${\G}({\bf V})$ is defined on \emph{kernel} objects derived from $p({\bf V})$ and \emph{conditional ADMGs} derived from $\G({\bf V})$.  The derivations are via a fixing operation, which can be causally interpreted as a single application of the g-formula on a single variable (to either a graph or a kernel) to obtain another graph or another kernel.


\subsubsection{Conditional Graphs And Kernels}
\label{subsubsec:cond}

A conditional acyclic directed mixed graph (CADMG) $\mathcal{G}(\mathbf{V}, \mathbf{W})$ is an ADMG in which the nodes are partitioned into $\mathbf{W}$, representing \textit{fixed variables}, and $\mathbf{V}$, representing \textit{random variables}. Only outgoing directed edges may be adjacent to variables in ${\bf W}$.

A \textit{kernel} $q_{\mathbf{V}}(\mathbf{V} | \mathbf{W})$ is a mapping from values in $\mathbf{W}$ to normalized densities over $\mathbf{V}$ \cite{lauritzen96graphical}.  In other words, kernels act like conditional distributions in the sense that $\sum_{\mathbf{v} \in \mathbf{V}} q_{\mathbf{V}}(\mathbf{v} | \mathbf{w}) = 1, \forall \mathbf{w} \in \mathbf{W}$. Conditioning and marginalization in kernels are defined in the usual way.  For $\mathbf{A} \subseteq \mathbf{V}$, we define $q(\mathbf{A} | \mathbf{W}) \equiv \sum_{\mathbf{V} \setminus \mathbf{A}} q(\mathbf{V} | \mathbf{W})$ and $q(\mathbf{V} \setminus \mathbf{A} | \mathbf{A}, \mathbf{W}) \equiv {q(\mathbf{V} | \mathbf{W})}/{q(\mathbf{A} | \mathbf{W})}$.

\subsubsection{Fixability And Fixing}
\label{subsubsec:fix}

A variable $V \in \mathbf{V}$ in a CADMG $\mathcal{G}$ is \textit{fixable} if $\de_{\G}(V) \cap \dis_{\G}(V) = \{ V \}$. In other words, $V$ is fixable if paths $V \leftrightarrow \dots \leftrightarrow U$ and $V \rightarrow \dots \rightarrow U$ do not \emph{both} exist in $\mathcal{G}$ for any $U \in \mathbf{V} \setminus \{ V \}$.
Given a CADMG $\G({\bf V},{\bf W})$ and $V \in {\bf V}$ fixable in $\G$, the fixing operator $\phi_V(\G)$ yields a new CADMG 
$\mathcal{G}'(\mathbf{V} \setminus \{V\}, \mathbf{W} \cup \{V\})$, where all edges with arrowheads into $V$ are removed, and all other edges in $\G$ are kept.  Similarly, given a CADMG $\G({\bf V},{\bf W})$, a kernel $q_{\bf V}({\bf V} | {\bf W})$, and $V \in {\bf V}$ fixable in $\G$, the fixing operator $\phi_V(q_{\bf V}; \G)$ yields a new kernel $q_{\mathbf{V} \setminus \{V\}}'(\mathbf{V} \setminus \{V\} | \mathbf{W} \cup \{V\}) \equiv \frac{q_{\mathbf{V}}(\mathbf{V} | \mathbf{W})}{q_{\mathbf{V}}(V | \nd_{\mathcal{G}}(V), \mathbf{W})}$. Fixing
is a probabilistic operation in which we divide a kernel by a conditional kernel. In some cases this operates as a conditioning operation, in other cases as a marginalization operation, and in yet other cases, as neither, depending on the structure of the kernel being divided.

For a set $\mathbf{S} \subseteq \mathbf{V}$ in a CADMG $\mathcal{G}$, if all vertices in ${\bf S}$ can be ordered into a sequence $\sigma_{\bf S} = \langle S_1, S_2, \dots \rangle$ such that $S_1$ is fixable in $\mathcal{G}$, $S_2$ in $\phi_{S_1}(\mathcal{G})$, etc., $\mathbf{S}$ is said to be \emph{fixable} in $\G$, ${\bf V} \setminus {\bf S}$ is said to be \emph{reachable} in $\G$, and $\sigma_{\bf S}$ is said to be valid.  A reachable set ${\bf C}$ is said to be \emph{intrinsic} if ${\cal G}_{\bf C}$ has a single district, where ${\cal G}_{\bf C}$ is the induced subgraph where we keep all vertices in ${\bf C}$ and edges whose endpoints are in ${\bf C}$.
We will define $\phi_{\sigma_{\bf S}}(\G)$ and $\phi_{\sigma_{\bf S}}(q_{\bf V}; \G)$ via the usual function composition to yield operators that fix all elements in ${\bf S}$ in the order given by  $\sigma_{\bf S}$. 

The distribution $p({\bf V})$ is said to obey the nested factorization for an ADMG $\G$ if there exists a set of kernels
$\left\{ q_{\bf C}\left({\bf C} \mid \pa_{\cal G}({\bf C})\right) \mid {\bf C} \text{ is intrinsic in }{\cal G} \right\}$ such that
for every fixable ${\bf S}$, and any valid $\sigma_{\bf S}$,
$\phi_{\sigma_{\bf S}}(p({\bf V});\G) = \prod_{{\bf D} \in {\cal D}(\phi_{\sigma_{\bf S}}(\G))} q_{\bf D}({\bf D} | \pa_{\G_{\bf S}}({\bf D}))$.
All valid fixing sequences for ${\bf S}$ yield the same CADMG $\G({\bf V} \setminus {\bf S}, {\bf S})$, and
if $p({\bf V})$ obeys the nested factorization for $\G$, all valid fixing sequences for ${\bf S}$ yield the same kernel.
As a result, for any valid sequence $\sigma$ for ${\bf S}$, we will redefine the operator $\phi_{\sigma}$, for both graphs and kernels, to be $\phi_{\bf S}$.  In addition, it can be shown that the above kernel set is characterized as:
$\left\{ q_{\bf C}\left({\bf C} \mid \pa_{\cal G}({\bf C})\right) \mid {\bf C} \text{ is intrinsic in }{\cal G} \right\} =
\left\{ \phi_{{\bf V} \setminus {\bf C}}(p\left({\bf V});{\cal G}\right) \mid {\bf C} \text{ is intrinsic in }{\cal G} \right\}$ \cite{richardson17nested}.
Thus, we can re-express the above nested factorization as stating that for any fixable set ${\bf S}$, we have
$\phi_{\bf S}(p({\bf V}); \G) = \prod_{{\bf D} \in {\cal D}(\phi_{{\bf S}}(\G))} \phi_{{\bf V} \setminus {\bf D}}(p({\bf V}); \G)$.

An important result in \cite{richardson17nested} states that if $p({\bf V} \cup {\bf H})$ obeys the factorization for a DAG $\G$ with vertex set ${\bf V} \cup {\bf H}$, then $p({\bf V})$ obeys the nested factorization for the latent projection ADMG $\G({\bf V})$.

\subsection{IDENTIFICATION AS A TRUNCATED NESTED FACTORIZATION}
\label{subsec:id}

For any disjoint subsets ${\bf Y},{\bf A}$ of ${\bf V}$ in a latent projection $\G({\bf V})$ representing a causal DAG $\G({\bf V}\cup{\bf H})$, define ${\bf Y}^* \equiv \an_{\G({\bf V})_{{\bf V} \setminus {\bf A}}}({\bf Y})$.
Then $p({\bf Y} ({\bf a}))$ is identified from $p({\bf V})$ in $\G$ \emph{if and only if} every set ${\bf D} \in {\cal D}(\G({\bf V})_{{\bf Y}^*})$ is intrinsic.  If identification holds, we have: 
{\small
\begin{align*}
p({\bf Y}({\bf a})) =
\sum_{{\bf Y}^* \setminus {\bf Y}} \prod_{{\bf D} \in {\cal D}(\G({\bf V})_{{\bf Y}^*})}
\!\!\!\!\!
\phi_{{\bf V} \setminus {\bf D}}(p({\bf V}); \G({\bf V})) \vert_{{\bf A} = {\bf a}}.
\end{align*}
}%
In other words, $p({\bf Y}({\bf a}))$ is identified if and only if it can be expressed as a factorization, where every piece corresponds to a kernel associated with a set intrinsic in $\G({\bf V})$.  Moreover, no term in this factorization contains elements of ${\bf A}$ as random variables, just as was the case in \eqref{eqn:g}. The above provides a concise formulation of the ID algorithm \cite{tian02general,shpitser06id} in terms of the nested Markov model which contains the causal model of the observed distribution.

If ${\bf Y} = \{ Y \}$, and ${\bf A} = \{ \pa_{\cal G}(Y) \}$, then the above truncated factorization has a simpler form:
{\small
\begin{align*}
p(Y({\bf a})) = \phi_{{\bf V} \setminus \{ Y \}}(p({\bf V}); {\cal G}) \vert_{{\bf A} = {\bf a}}.
\end{align*}
}%
In words, to identify the interventional distribution of $Y$ where all parents (direct causes) ${\bf A}$ of $Y$ are set to values ${\bf a}$, we must find a total ordering on variables other than $Y$ (${\bf V} \setminus \{ Y \}$) that forms a valid fixing sequence.  If such an ordering exists, the identifying functional is found from $p({\bf V})$ by applying the fixing operator to each variable in succession, in accordance with this ordering. Fig.~\ref{fig:totalorder} shows the identification of the functional $p(Y(a))$ following a total ordering of fixing $M, B, A$.

\begin{figure}[t]
	\begin{center}
		\scalebox{0.7}{
			\begin{tikzpicture}[>=stealth, node distance=1.2cm]
			\tikzstyle{format} = [thick, circle, minimum size=1.0mm, inner sep=0pt]
			\tikzstyle{square} = [draw, thick, minimum size=1.0mm, inner sep=3pt]
			\begin{scope}
			\path[->, very thick]
			node[] (b) {$B$}
			node[right of=b] (m) {$M$}
			node[right of=m] (a) {$A$}
			node[right of=a] (y) {$Y$}
			(b) edge[blue] (m)
			(m) edge[blue] (a)
			(a) edge[blue] (y)
			(b) edge[red, <->, bend left] (a)
			(b) edge[red, <->, bend right] (y)
			node [below of=m, xshift=-0.2cm, yshift=0.0cm] (l) {(a) $p({\bf V; \mathcal{G}})$}
			;
			\end{scope}
			\begin{scope}[xshift=6cm]
			\path[->, very thick]
			node[] (b) {$B$}
			node[square, right of=b] (m) {$m$}
			node[right of=m] (a) {$A$}
			node[right of=a] (y) {$Y$}
			(m) edge[blue] (a)
			(a) edge[blue] (y)
			(b) edge[red, <->, bend left] (a)
			(b) edge[red, <->, bend right] (y)
			node [below of=a, xshift=-1.1cm ,yshift=0.0cm] (l) {(b) $\phi_{\{M\}}(p({\bf V; \mathcal{G}}))$}
			node [below of=l, xshift=0.77cm, yshift=0.4cm] {$= \ p(Y,A|M,B)\ p(B)$}
			;
			\end{scope}
			\begin{scope}[yshift=-3cm]
			\path[->, very thick]
			node[square] (b) {$b$}
			node[square, right of=b] (m) {$m$}
			node[right of=m] (a) {$A$}
			node[right of=a] (y) {$Y$}
			(m) edge[blue] (a)
			(a) edge[blue] (y)
			node [below of=m, xshift=0.5cm, yshift=0.2cm] (l) {(c) $\phi_{\{M, B\}}(p({\bf V; \mathcal{G}})$}
			node [below of=l, xshift=1cm, yshift=0.4cm] {$= \ \sum_{B} \ p(Y,A|M,B) \ p(B)$}
			;
			\end{scope}
			\begin{scope}[xshift=6cm, yshift=-3cm]
			\path[->, very thick]
			node[square] (b) {$b$}
			node[square, right of=b] (m) {$m$}
			node[square, right of=m] (a) {$a$}
			node[right of=a] (y) {$Y$}
			(a) edge[blue] (y)
			node [below of=m, xshift=0.45cm, yshift=0.2cm] (l) {(d) $\phi_{\{M, B, A\}}(p({\bf V; \mathcal{G}}))$}
			node [below of =l, xshift=0.65cm, yshift=0.4cm] {$= \ \frac{\sum_{B} \ p(Y,A=a|M,B) \ p(B)}{\sum_{B}p(A=a|M,B) \ p(B)}$}
			;
			\end{scope}
			\end{tikzpicture}
		}
	\end{center}
	\caption{Identification of $p(Y(a))$ by following a total order of valid fixing operations.}
	\label{fig:totalorder}
\end{figure}
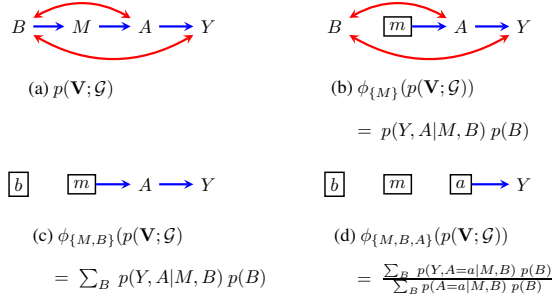

Before generalizing these tools to the identification of missing data models, we first introduce the representation of these models using DAGs.

\section{MISSING DATA MODELS OF A DAG}
\label{sec:missing}

Missing data models are sets of full data laws (distributions) $p({\bf X}^{({\bf 1})}, {\bf O}, {\bf R})$ composed of the target laws $p({\bf X}^{({\bf 1})}, {\bf O})$, and the nuisance laws $p({\bf R} | {\bf X}^{({\bf 1})}, {\bf O})$ defining the missingness processes.  The target law is over a set ${\bf X}^{({\bf 1})} \equiv \{ X_1^{(1)}, \ldots, X_k^{(1)} \}$ of random variables that are potentially missing, and a set ${\bf O} \equiv \{ O_1, \ldots, O_m \}$ of random variables that are always observed.  The nuisance law defines the behavior of missingness indicators ${\bf R} \equiv \{ R_1, \ldots, R_k \}$ given values of missing and observed variables.  Each missing variable $X^{(1)}_i \in {\bf X}^{({\bf 1})}$ has a corresponding observed proxy variable $X_i$, defined as $X_i \equiv X_i^{(1)}$ if $R_i = 1$, and defined as $X_i \equiv ``?"$ if $R_i = 0$ (this is the missing data analogue of the consistency property in causal inference).
As a result, the observed data law in missing data problems is $p({\bf R},{\bf O},{\bf X})$, while some function of the target law $p({\bf X}^{({\bf 1})},{\bf O})$, as its name implies, is the target of inference.  The goal in missing data problems is to estimate the latter from the former.
By chain rule of probability,
{\small
\begin{align}
p({\bf X}^{({\bf 1})},{\bf O}) = \frac{p({\bf X},{\bf O},{\bf R}={\bf 1})}{p({\bf R}={\bf 1}|{\bf X}^{(1)},{\bf O})}.
\label{eqn:chain}
\end{align}
}%
In other words, $p({\bf X}^{({\bf 1})},{\bf O})$ is identified from the observed data law $p({\bf R}, {\bf O}, {\bf X})$ if and only if $p({\bf R}={\bf 1}|{\bf X}^{({\bf 1})},{\bf O})$ is.  In general, $p({\bf X}^{({\bf 1})})$ is not identified from the observed data law, unless sufficient restrictions are placed on the full data law defining the missing data model. 


Many popular missing data models may be represented as a factorization of the full data law with respect to a DAG \cite{mohan2013missing}.  These include the permutation model, the monotone MAR model, the block sequential MAR model, and certain submodels of the no self-censoring model \cite{robins97non, shpitser2016consistent, zhou2010block}.

Given a set of full data laws $p({\bf X}^{({\bf 1})}, {\bf O}, {\bf R})$, a DAG ${\cal G}$ with the following properties may be used to represent a missing data model: ${\cal G}$ has a vertex set ${\bf X}^{({\bf 1})}, {\bf O}, {\bf R},{\bf X}$; for each $X_i \in {\bf X}$, $\pa_{\cal G}(X_i) = \{ R_i, X_i^{(1)} \}$; for each $R_i \in {\bf R}$, $\de_{\cal G}(R_i) \cap ({\bf X}^{(1)} \cup {\bf O}) = \emptyset$.
Given a DAG ${\cal G}$ with the above properties, a missing data model associated with ${\cal G}$ is the set of distributions
$p({\bf X}^{({\bf 1})}, {\bf O}, {\bf R})$ that can be written as
{\small
\begin{align}
\prod_{X_i \in {\bf X}} p(X_i | R_i, X_i^{(1)}) \prod_{V \in {\bf X}^{(1)}\cup{\bf O}\cup{\bf R}} p(V | \pa_{\cal G}(V)),
\label{eqn:miss-fact}
\end{align}
}%
where the set of factors of the form $p(X_i | R_i, X_i^{(1)})$ are deterministic to remain consistent with the definition of $X_i$.
Note that by standard results on DAG models, conditional independences in $p({\bf X}^{({\bf 1})}, {\bf O}, {\bf R})$ may be read off from ${\cal G}$ by the d-separation criterion \cite{pearl88probabilistic}.

\section{EXAMPLES OF IDENTIFIED MODELS}
\label{sec:examples}

In this section, we describe a set of examples of missing data models that factorize as in \eqref{eqn:miss-fact} for different DAGs, where the target law is identified.  We start with simpler examples where sequential fixing techniques from causal inference suffice to obtain identification, then move on to describe more complex examples where existing algorithms in the literature suffice, and finally proceed to examples where no published method known to us obtains identification, illustrating an identifiability gap in existing methods.  In these examples, we show how identification may be obtained by appropriately generalizing existing techniques.  In these discussions, we concentrate on obtaining identification of the nuisance law $p({\bf R} | {\bf X}^{({\bf 1})},{\bf O})$ evaluated at ${\bf R} = {\bf 1}$, as this suffices to identify the target law $p({\bf X}^{({\bf 1})}, {\bf O})$ by \eqref{eqn:chain}.  In the course of describing these examples, we will obtain intermediate graphs and kernels.  In these graphs, lower case letters (e.g. $v$) indicates the variable $V$ is evaluated at $v$ (for $R_i, r_i=1$). A square vertex indicates $V$ had been fixed. Drawing the vertex normally with lower case indicates $V$ was conditioned on (creating selection bias in the subproblem). For brevity, we use $\bias{R_i}$ to denote $\{R_i=1 \}$.

\begin{figure*}[ht]
	\begin{center}
		\scalebox{0.75}{
			\begin{tikzpicture}[>=stealth, node distance=1.5cm]
			\tikzstyle{format} = [thick, circle, minimum size=1.0mm, inner sep=0pt]
			\tikzstyle{square} = [draw, thick, minimum size=1.0mm, inner sep=3pt]
			\begin{scope}
			\path[->, very thick]
			node[] (x11) {$X^{(1)}_1$}
			node[right of=x11] (x21) {$X^{(1)}_2$}
			node[right of=x21] (x31) {$X^{(1)}_3$}
			node[below of=x11] (r1) {$R_1$}		
			node[below of=x21] (r2) {$R_2$}
			node[below of=x31] (r3) {$R_3$}
			node[below of=r1] (x1) {$X_1$}
			node[below of=r2] (x2) {$X_2$}
			node[below of=r3] (x3) {$X_3$}
			(x11) edge[blue] (x21)
			(x11) edge[blue] (r3)
			(x21) edge[blue] (x31)
			(x11) edge[blue, bend left] (x31)
			(x11) edge[blue] (r2)
			(x21) edge[blue] (r3)
			(r1) edge[blue] (r2)
			(r2) edge[blue] (r3)
			(r1) edge[blue, bend right] (r3)
			(r1) edge[gray] (x1)
			(r2) edge[gray] (x2)
			(r3) edge[gray] (x3)
			(x11) edge[gray, bend right] (x1)
			(x21) edge[gray, bend right] (x2)
			(x31) edge[gray, bend left] (x3)
			node [below of=x2, yshift=0.7cm] {(a) $\mathcal{G}$}
			;
			\end{scope}
			\begin{scope}[xshift=5.5cm]
			\path[->, very thick]
			node[] (x11) {$X_1$}
			node[right of=x11] (x21) {$X^{(1)}_2$}
			node[right of=x21] (x31) {$X^{(1)}_3$}
			node[square, below of=x11] (r1) {$r_1$}		
			node[below of=x21] (r2) {$R_2$}
			node[below of=x31] (r3) {$R_3$}
			node[below of=r2] (x2) {$X_2$}
			node[below of=r3] (x3) {$X_3$}
			(x11) edge[blue] (x21)
			(x11) edge[blue] (r3)
			(x21) edge[blue] (x31)
			(x11) edge[blue, bend left] (x31)
			(x11) edge[blue] (r2)
			(x21) edge[blue] (r3)
			(r1) edge[blue] (r2)
			(r2) edge[blue] (r3)
			(r1) edge[blue, bend right] (r3)
			(r2) edge[gray] (x2)
			(r3) edge[gray] (x3)
			(x21) edge[gray, bend right] (x2)
			(x31) edge[gray, bend left] (x3)
			node [below of=x2, yshift=0.7cm] {(b) $\G_1 \equiv \phi_{R_1}(\mathcal{G})$}
			;
			\end{scope}
			\begin{scope}[xshift=11cm]
			\path[->, very thick]
			node[] (x11) {$X_1$}
			node[right of=x11] (x21) {$X_2$}
			node[right of=x21] (x31) {$X^{(1)}_3$}
			node[square, below of=x11] (r1) {$r_1$}		
			node[square, below of=x21] (r2) {$r_2$}
			node[below of=x31] (r3) {$R_3$}
			node[below of=r2] (x2) {}
			node[below of=r3] (x3) {$X_3$}
			(x11) edge[blue] (x21)
			(x11) edge[blue] (r3)
			(x21) edge[blue] (x31)
			(x11) edge[blue, bend left] (x31)
			(x21) edge[blue] (r3)
			(r2) edge[blue] (r3)
			(r1) edge[blue, bend right] (r3)
			(r3) edge[gray] (x3)
			(x31) edge[gray, bend left] (x3)
			node [below of=x2, yshift=0.7cm] {(c) $\G_2 \equiv  \phi_{R_2}(\mathcal{G}_1)$}
			;
			\end{scope}
			\begin{scope}[xshift=16.5cm]
			\path[->, very thick]
			node[] (x11) {$X^{(1)}_1$}
			node[right of=x11] (x21) {$X^{(1)}_2$}
			node[right of=x21] (x31) {$X^{(1)}_3$}
			node[below of=x11] (r1) {$R_1$}		
			node[below of=x21] (r2) {$R_2$}
			node[below of=x31] (r3) {$R_3$}
			node[below of=r1] (x1) {$X_1$}
			node[below of=r2] (x2) {$X_2$}
			node[below of=r3] (x3) {$X_3$}
			(x11) edge[blue] (x21)
			(x11) edge[blue] (r2)
			(x21) edge[blue] (x31)
			(x11) edge[blue, bend left] (x31)
			(x31) edge[blue] (r2)
			(x21) edge[blue] (r3)
			(x21) edge[blue] (r1)
			(r1) edge[gray] (x1)
			(r2) edge[gray] (x2)
			(r3) edge[gray] (x3)
			(x31) edge[blue] (r1)
			(x11) edge[blue] (r3)
			(x11) edge[gray, bend right] (x1)
			(x21) edge[gray, bend right] (x2)
			(x31) edge[gray, bend left] (x3)
			node [below of=x2, yshift=0.7cm] {(d)}
			;
			\end{scope}
			\end{tikzpicture}
		}
	\end{center}
	\caption{(a), (b), (c) are intermediate graphs obtained in identification of a block-sequential model by fixing $\{R_1, R_2, R_3\}$ in sequence. (d) is an MNAR model that is identifiable by fixing all $R$s in parallel.}
	\label{fig:blockseq}
\end{figure*}
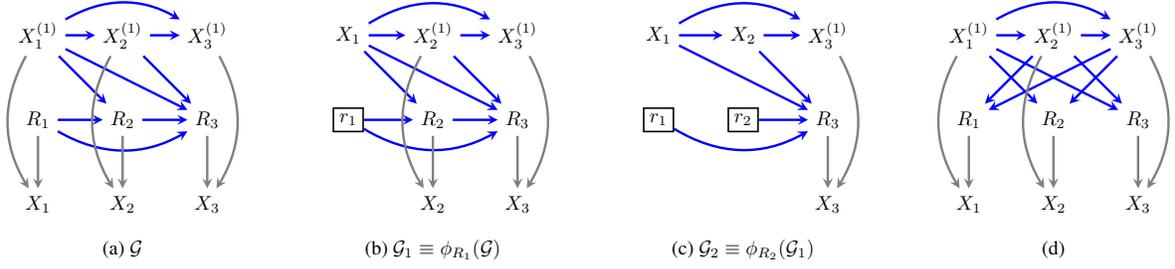

We first consider the block-sequential MAR model \cite{zhou2010block}, shown in Fig.~\ref{fig:blockseq} for three variables. The target law is identified by applying the (valid) fixing sequence $\langle R_1, R_2, R_3 \rangle$ via the operator $\phi$ to ${\cal G}$ and $p({\bf R},{\bf X})$.  We proceed as follows.  $p(R_1 | \pa_{\cal G}(R_1)) = p(R_1 | \nd_{\cal G}(R_1)) = p(R_1)$ is identified immediately.  Applying the fixing operator $\phi_{R_1}$ yields the graph ${\cal G}_1 \equiv \phi_{R_1}({\cal G})$ shown in Fig.~\ref{fig:blockseq}(b), and corresponding kernel
$q_1(X_1^{(1)}, X_2, X_3, R_2, R_3 | \bias{R_1}) \equiv p(X_1, X_2, X_3, R_2, R_3, \bias{R_1}) / p(\bias{R_1})$ where $X_1^{(1)}$ is now observed.
Thus, in the new subproblem represented by ${\cal G}_1$ and $q_1$, $p(R_2 | \pa_{\cal G}(R_2)) \vert_{{\bf R=1}} = q_1(R_2 | X_1^{(1)}, \bias{R_1})$ is identified.  Applying the fixing operator $\phi_{R_2}$ to ${\cal G}_1$ and $q_1$ yields ${\cal G}_2 \equiv \phi_{R_2}({\cal G}_1)$ shown in Fig.~\ref{fig:blockseq}(c), and $q_2(X_1^{(1)}, X_2^{(1)}, X_3, R_3 | \bias{R_1, R_2}) = q_1(X_1^{(1)}, X_2, X_3, R_2, R_3 | \bias{R_1}) / q_1(R_2 | X_1^{(1)}, \bias{R_1})$.  Finally, in the new subproblem represented by ${\cal G}_2$ and $q_2$, $p(R_3 | \pa_{\cal G}(R_3)) \vert_{{\bf R = 1}} = q_2(R_3 | X_1^{(1)},X_2^{(1)},\bias{R_1, R_2})$ is identified.  Applying the fixing operator $\phi_{R_3}$ to ${\cal G}_2$ and $q_2$ yields $q_3(X_1^{(1)}, X_2^{(1)}, X_3^{(1)} | \bias{R_1, R_2, R_3}) = p(X_1^{(1)}, X_2^{(1)}, X_3^{(1)})$.  The identifying functional for the target law only involves monotone cases (cases where $R_i = 0$ implies $R_{i+1} = 0$) just as would be the case under the monotone MAR model, although this model does not assume monotonicity and is not MAR.  In this simple example, identification may be achieved purely by causal inference methods, by treating variables in ${\bf R}$ as treatments, and finding a valid fixing sequence on them.  In this example, each $R_i$ in the sequence is fixable given that the previous variables are fixable, since all parents of each $R_i$ become observed at the time it is fixed.

Following a total order to fix is not always sufficient to identify the target law, as noted in \cite{mohan2013missing, mohan14missing, shpitser15missing}. Consider the model represented by DAG in Fig.~\ref{fig:blockseq}(d).  For any $R_i$ in this model, say $R_1$, we have, by d-separation, that $p(R_1 | \pa_{\cal G}(R_1)) = p(R_1 | X_2^{(1)}, X_3^{(1)}, \bias{R_2, R_3})$, which is identified.  However, if we were to fix $R_1$ in $p({\bf X},{\bf R})$, we would obtain a kernel
$q_1(X_1^{(1)}, X_2, X_3, \bias{R_2, R_3} | \bias{R_1})$ where selection bias on $R_2$ and $R_3$ is introduced.  The fact that $q_1$ is not available at all levels of $R_2$ and $R_3$ prevents us from sequentially obtaining $p(R_i | \pa_{\cal G}(R_i))$, for $R_i = R_2,R_3$, due to our inability to sum out those variables from $q_1$.

The model in Fig.~\ref{fig:blockseq}(d) allows identification of the target law in another way, however.  This follows from the fact that
$p(R_i | \pa_{\cal G}(R_i))$ is identified for each $R_i$ by exploiting conditional independences in $p({\bf X},{\bf R})$ displayed by Fig.~\ref{fig:blockseq}(d).  Since $p({\bf R} | {\bf X}^{({\bf 1})}) = \prod_{i=1}^3 p(R_i | \pa_{\cal G}(R_i))$, the nuisance law is identified, which means the target law is also identified, as long as we fix $R_1,R_2,R_3$ \emph{in parallel} (as in (\ref{eqn:chain})) rather than sequentially.  In other words, the model is identified, but no total order on fixing operations suffices for identification.  A general algorithm that aimed to fix indicators in ${\bf R}$ in parallel, while potentially exploiting causal inference fixing operations to identify each $p(R_i | \pa_{\cal G}(R_i))$
was proposed in \cite{shpitser15missing}. Our subsequent examples show that this algorithm is insufficient to obtain identification of the target law in general, and thus is incomplete.

Consider the DAG in Fig.~\ref{fig:incomplete}.  Since $R_2$ is a child of $R_3$ and $X_2^{(1)}$ is a parent of $R_3$, we cannot obtain $p(R_3 | \pa_{\cal G}(R_3)) = p(R_3 | X_2^{(1)})$ by d-separation in any kernel (including the original distribution) where $R_2$ is not fixed.  Thus, any total order on fixing operations of elements in ${\bf R}$ must start with $R_1$ or $R_2$.  Fixing either of these variables entails dividing $p({\bf X},{\bf R})$ by some factor $p(R_i | \pa_{\cal G}(R_i))$, which is identified as either $p(R_1 | X_3^{(1)}, \bias{R_3})$ or $p(R_2 | X_1^{(1)}, \bias{R_1})$.  This division entails inducing selection bias on the subsequent kernel $q_1$ for a variable not yet fixed (either $R_3$ or $R_1$).  Thus, no total order on fixing operations works to identify the target law in this model.  At the same time, attempting to fix all $R$ variables in parallel would fail as well, since we cannot identify $p(R_3 | X_2^{(1)})$ either in the original distribution or any kernel obtained by standard causal inference operations described in \cite{shpitser15missing}.  In particular, in any such kernel or distribution $R_3$ remains dependent on $R_2$ given $X_2^{(1)}$.

However, the target law in this model is identified by following a \emph{partial order} $\prec$ of fixing operations.  In this partial order, $R_1$ is incompatble with $R_2$, and $R_2 \prec R_3$. This results in an identification strategy where we fix each variable \emph{only} given that variables earlier than it in the \emph{partial} order are fixed.  That is, distributions $p(R_1 | X^{(1)}_3) = p(R_1 | X_3, \bias{R_3})$ and $p(R_2 | X^{(1)}_1, R_3) = p(R_2 | X_1, \bias{R_1}, R_3)$ are obtained directly in the original distribution without fixing anything.  The distribution $p(R_3 | \pa_{\cal G}(R_3))$, on the other hand, is obtained in the kernel $q_1(X_1, X_2^{(1)}, X_3, \bias{R_1}, R_3 | \bias{R_2}) = p({\bf X},{\bf R}) / p(R_2 | X_1, \bias{R_1}, R_3)$ after $R_2$ (the variable earlier than $R_3$ in the partial order) is fixed.  The graph corresponding to this kernel is shown in Fig.~\ref{fig:incomplete}(b).  Note that in this graph $X_2^{(1)}$ is observed, and there is selection bias on $R_1$.
However, it easily follows by d-separation that $R_3$ is independent of $R_1$.  It can thus be shown that $p(R_3 | X_2^{(1)}) =
q_1(R_3 | X_2^{(1)}, \bias{R_2})$ even if $q_1$ is only available at value $R_1 = 1$.  Since all $p(R_i | \pa_{\cal G}(R_i))$ are identified, so is the target law in this model, by (\ref{eqn:chain}).
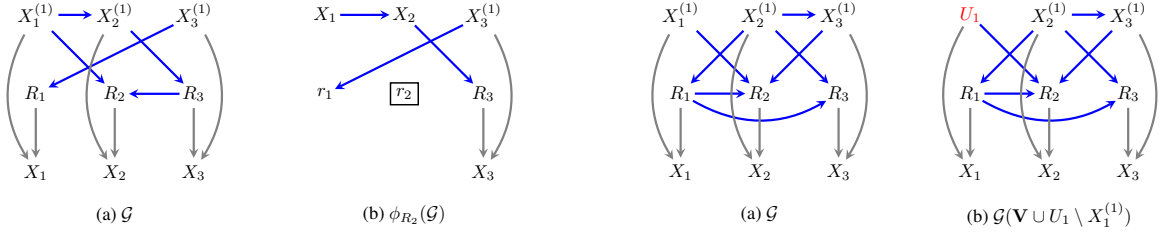
\begin{figure}[t]
\begin{center}
\scalebox{0.7}{
\begin{tikzpicture}[>=stealth, node distance=1.5cm]
    \tikzstyle{format} = [thick, circle, minimum size=1.0mm, inner sep=0pt]
	\tikzstyle{square} = [draw, thick, minimum size=1.0mm, inner sep=3pt]
	\begin{scope}
		\path[->, very thick]
			node[format] (x11) {$X^{(1)}_1$}
			node[format, right of=x11] (x21) {$X^{(1)}_2$}
			node[format, right of=x21] (x31) {$X^{(1)}_3$}
			node[format, below of=x11] (r1) {$R_1$}
			node[format, below of=x21] (r2) {$R_2$}
			node[format, below of=x31] (r3) {$R_3$}
			node[format, below of=r1] (x1) {$X_1$}
			node[format, below of=r2] (x2) {$X_2$}
			node[format, below of=r3] (x3) {$X_3$}
			(x11) edge[blue] (x21)
			(x11) edge[blue] (r2)
			(x21) edge[blue] (r3)
			(r3) edge[blue] (r2)
			(x31) edge[blue] (r1)
			(r1) edge[gray] (x1)
			(r2) edge[gray] (x2)
			(r3) edge[gray] (x3)
			(x11) edge[gray, bend right] (x1)
			(x21) edge[gray, bend right] (x2)
			(x31) edge[gray, bend left] (x3)
			node [below of=x2, yshift=0.7cm] {(a) $\mathcal{G}$}
			;
	\end{scope}
	\begin{scope}[xshift=5.5cm, yshift=0cm]
		\path[->, very thick]
			node[format] (x11) {$X_1$}
			node[format, right of=x11] (x21) {$X_2$}
			node[format, right of=x21] (x31) {$X^{(1)}_3$}
			node[format, below of=x11] (r1) {$r_1$}
			node[square, below of=x21] (r2) {$r_2$}
			node[format, below of=x31] (r3) {$R_3$}
			node[format, below of=r2] (x2) {}
			node[format, below of=r3] (x3) {$X_3$}
			(x11) edge[blue] (x21)
			(x21) edge[blue] (r3)
			(x31) edge[blue] (r1)
			(r3) edge[gray] (x3)
			(x31) edge[gray, bend left] (x3)
			node [below of=x2, yshift=0.7cm] {(b) $\phi_{R_2}(\mathcal{G})$}
			;
	\end{scope}
\end{tikzpicture}
}
\end{center}
\caption{(a) A DAG where $R$s are fixed according to a partial order. (b) The CADMG obtained by fixing $R_2$.}
\label{fig:incomplete}
\end{figure}

Next, we consider the model in Fig.~\ref{fig:fakeprop}. Here, $p(R_2 | X_1^{(1)}, X_3^{(1)}, R_1) = p(R_2 | X_1, X_3, \bias{R_1, R_3})$ and
$p(R_3 | X_2^{(1)}, R_1) = p(R_3 | X_2, \bias{R_2}, R_1)$ are identified immediately.  However, $p(R_1 | X_2^{(1)})$ poses a problem.
In order to identify this distribution, we either require that $R_1$ is conditionally independent of $R_2$, possibly after some fixing operations, or we are able to render $X^{(1)}_2$ observable by fixing $R_2$ in some way.  Neither seems to be possible in the problem as stated.  In particular, fixing $R_2$ via dividing by $p(R_2 | X_1^{(1)}, X_3^{(1)}, R_1)$ will necessarily induce selection bias on $R_1$, which will prevent identification of $p(R_1 | X_2^{(1)})$ in the resulting kernel.

However, we can circumvent the difficulty by treating $X_1^{(1)}$ as an \emph{unobserved variable} $U_1$, and attempting the problem in the resulting (hidden variable) DAG shown in Fig.~\ref{fig:fakeprop}(b), and its latent projection ADMG $\tilde{\cal G}$ shown in Fig.~\ref{fig:fakeprop}(c), where $U_1$ is ``projected out.''  In the resulting problem, we can fix variables according to a partial order $\prec$
where $R_2$ and $R_3$ are incompatible, $R_2 \prec R_1$, and $R_3 \prec R_1$.
Thus, we are able to fix $R_2$ and $R_3$ in parallel by dividing by $p(R_2 | \mb_{\tilde{\cal G}}(R_2)) = p(R_2 | X_1, R_1, X_3^{(1)}, \bias{R_3})$ and $p(R_3 | R_1, X_2^{(1)}) = p(R_3 | R_1, X_2, \bias{R_2})$, leading to a kernel
$\tilde{q}_1(X_1, X_2^{(1)}, X_3^{(1)}, R_1 | \bias{R_2, R_3})$, and the graph $\phi_{\prec R_1}(\tilde{\cal G})$ shown in Fig.~\ref{fig:fakeprop}(d), where notation $\phi_{\prec R_1}$ means ``fix all necessary elements that occur earlier than $R_1$ in the partial order, in a way consistent with that partial order.''  In this example, this means fixing $R_2$ and $R_3$ in parallel.  We will describe how fixing operates given general \emph{fixing schedules} given by a partial order later in the paper.  In the kernel $\tilde{q}_1$ the parent of $R_1$
is observed data, meaning that $p(R_1 | X_2^{(1)})$ is identified as $\tilde{q}_1(R_1 | X_2, \bias{R_2, R_3})$.  This implies the target law is identified in this model.

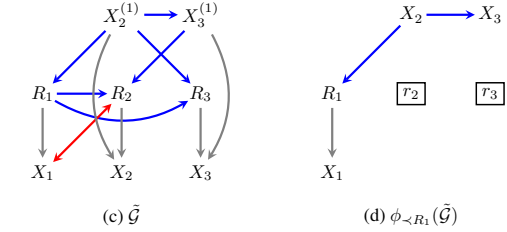
\begin{figure}[t]
\begin{center}
\scalebox{0.7}{
\begin{tikzpicture}[>=stealth, node distance=1.5cm]
    \tikzstyle{format} = [thick, circle, minimum size=1.0mm, inner sep=0pt]
	\tikzstyle{square} = [draw, thick, minimum size=1.0mm, inner sep=3pt]
	\begin{scope}
		\path[->, very thick]
			node[format] (x11) {$X^{(1)}_1$}
			node[format, right of=x11] (x21) {$X^{(1)}_2$}
			node[format, right of=x21] (x31) {$X^{(1)}_3$}
			node[format, below of=x11] (r1) {$R_1$}
			node[format, below of=x21] (r2) {$R_2$}
			node[format, below of=x31] (r3) {$R_3$}
			node[format, below of=r1] (x1) {$X_1$}
			node[format, below of=r2] (x2) {$X_2$}
			node[format, below of=r3] (x3) {$X_3$}
			(x11) edge[blue] (r2)
			(x21) edge[blue] (r1)
			(x21) edge[blue] (r3)
			(x21) edge[blue] (x31)
			(x31) edge[blue] (r2)
			(r1) edge[blue] (r2)
			(r1) edge[blue, bend right] (r3)
			(r1) edge[gray] (x1)
			(r2) edge[gray] (x2)
			(r3) edge[gray] (x3)
			(x11) edge[gray, bend right] (x1)
			(x21) edge[gray, bend right] (x2)
			(x31) edge[gray, bend left] (x3)
			node [below of=x2, yshift=0.7cm] {(a) $\mathcal{G}$}
			;
	\end{scope}
	\begin{scope}[xshift=5.5cm]
		\path[->, very thick]
			node[format] (x11) {$\red{U_1}$}
			node[format, right of=x11] (x21) {$X^{(1)}_2$}
			node[format, right of=x21] (x31) {$X^{(1)}_3$}
			node[format, below of=x11] (r1) {$R_1$}
			node[format, below of=x21] (r2) {$R_2$}
			node[format, below of=x31] (r3) {$R_3$}
			node[format, below of=r1] (x1) {$X_1$}
			node[format, below of=r2] (x2) {$X_2$}
			node[format, below of=r3] (x3) {$X_3$}
			(x11) edge[blue] (r2)
			(x21) edge[blue] (r1)
			(x21) edge[blue] (r3)
			(x21) edge[blue] (x31)
			(x31) edge[blue] (r2)
			(r1) edge[blue] (r2)
			(r1) edge[blue, bend right] (r3)
			(r1) edge[gray] (x1)
			(r2) edge[gray] (x2)
			(r3) edge[gray] (x3)
			(x11) edge[gray, bend right] (x1)
			(x21) edge[gray, bend right] (x2)
			(x31) edge[gray, bend left] (x3)
			node [below of=x2, yshift=0.7cm] {(b) $\mathcal{G}({\bf V} \cup U_1 \setminus X^{(1)}_1)$}
			;
	\end{scope}
	\begin{scope}[yshift=-5cm]
		\path[->, very thick]
			node[format] (x11) {}
			node[format, right of=x11] (x21) {$X^{(1)}_2$}
			node[format, right of=x21] (x31) {$X^{(1)}_3$}
			node[format, below of=x11] (r1) {$R_1$}
			node[format, below of=x21] (r2) {$R_2$}
			node[format, below of=x31] (r3) {$R_3$}
			node[format, below of=r1] (x1) {$X_1$}
			node[format, below of=r2] (x2) {$X_2$}
			node[format, below of=r3] (x3) {$X_3$}
			(x21) edge[blue] (r1)
			(x21) edge[blue] (r3)
			(x21) edge[blue] (x31)
			(x31) edge[blue] (r2)
			(r1) edge[blue] (r2)
			(r1) edge[blue, bend right] (r3)
			(r1) edge[gray] (x1)
			(r2) edge[gray] (x2)
			(r3) edge[gray] (x3)
			(x1) edge[red, <->] (r2)
			(x21) edge[gray, bend right] (x2)
			(x31) edge[gray, bend left] (x3)
			node [below of=x2, yshift=0.7cm] {(c) $\tilde{\mathcal{G}}$}
			;
	\end{scope}
	\begin{scope}[xshift=5.5cm, yshift=-5cm]
		\path[->, very thick]
			node[format] (x11) {}
			node[format, right of=x11] (x21) {$X^{}_2$}
			node[format, right of=x21] (x31) {$X^{}_3$}
			node[format, below of=x11] (r1) {$R_1$}
			node[square, below of=x21] (r2) {$r_2$}
			node[square, below of=x31] (r3) {$r_3$}
			node[format, below of=r1] (x1) {$X_1$}
			node[ below of=r2] (x2) {}
			(x21) edge[blue] (r1)
			(x21) edge[blue] (x31)
			(r1) edge[gray] (x1)
			node [below of=x2, yshift=0.7cm] {(d) $\phi_{\prec R_1}(\tilde{\mathcal{G}})$}
			;
	\end{scope}
\end{tikzpicture}
}
\end{center}
\caption{A DAG where selection bias on $R_1$ is avoidable by following a partial order fixing schedule on an ADMG induced by latent projecting out $X^{(1)}_1$.}
\label{fig:fakeprop}
\end{figure}

In general, to identify $p(R_i | \pa_{\cal G}(R_i))$, we may need to use separate partial fixing orders on different sets of variables for different $R_i \in {\bf R}$.  In addition, the fact that fixing introduces selection bias sometimes results in having to divide by a kernel where a \emph{set} of variables are random, something that was never necessary in causal inference problems.
In general, for a given $R_i$, the goal of a fixing schedule is to arrive at a kernel where an independence exists allowing us to identify $p(R_i | \pa_{\cal G}(R_i))$, even if some elements of $\pa_{\cal G}(R_i)$ are in ${\bf X}^{({\bf 1})}$ in the original problem.  This fixing must be given by a partial order, and sometimes on sets of variables.  In addition, some elements of ${\bf X}^{({\bf 1})}$ must be treated as hidden variables.  These complications are necessary in general to avoid creating selection bias in subproblems, and ultimately to identify the nuisance law.  The following example is a good illustration.


\begin{figure}[t]
	\begin{center}
		\scalebox{0.65}{
			\begin{tikzpicture}[>=stealth, node distance=1.5cm]
			\tikzstyle{format} = [thick, circle, minimum size=1.0mm, inner sep=0pt]
			\tikzstyle{square} = [draw, thick, minimum size=1.0mm, inner sep=3pt]
			\begin{scope}
			\path[->, very thick]
			node[format] (x11) {$X^{(1)}_1$}
			node[format, right of=x11] (x21) {$X^{(1)}_2$}
			node[format, right of=x21] (x31) {$X^{(1)}_3$}
			node[format, right of=x31] (x41) {$X^{(1)}_4$}
			node[format, below of=x11] (r1) {$R_1$}
			node[format, below of=x21] (r2) {$R_2$}
			node[format, below of=x31] (r3) {$R_3$}
			node[format, below of=x41] (r4) {$R_4$}
			node[format, below of=r1] (x1) {$X_1$}
			node[format, below of=r2] (x2) {$X_2$}
			node[format, below of=r3] (x3) {$X_3$}
			node[format, below of=r4] (x4) {$X_4$}
			(x11) edge[blue, bend left] (x31)
			(x11) edge[blue] (r2)
			(x31) edge[blue] (x21)
			(x21) edge[blue] (r1)
			(x11) edge[blue] (r4)
			(x41) edge[blue] (r1)
			(x41) edge[blue] (r3)
			(r2) edge[blue] (r1)
			(r2) edge[blue] (r3)
			(r3) edge[blue, bend left] (r1)
			(r4) edge[blue, bend left] (r2)
			(r1) edge[gray] (x1)
			(r2) edge[gray] (x2)
			(r3) edge[gray] (x3)
			(r4) edge[gray] (x4)
			(x11) edge[gray, bend right] (x1)
			(x21) edge[gray, bend right] (x2)
			(x31) edge[gray, bend left] (x3)
			(x41) edge[gray, bend left] (x4)
			node [below of=x2, xshift=0.75cm, yshift=0.7cm] {(a)}
			;
			\end{scope}
			\begin{scope}[xshift=6.5cm]
			\path[->, very thick]
			node[format] (x11) {$X^{(1)}_1$}
			node[format, right of=x11] (x21) {}
			node[format, right of=x21] (x31) {$X^{(1)}_3$}
			node[format, right of=x31] (x41) {}
			node[format, below of=x11] (r1) {$R_1$}
			node[format, below of=x21] (r2) {$R_2$}
			node[format, below of=x31] (r3) {$R_3$}
			node[format, below of=x41] (r4) {$R_4$}
			node[format, below of=r1] (x1) {$X_1$}
			node[format, below of=r2] (x2) {$X_2$}
			node[format, below of=r3] (x3) {$X_3$}
			node[format, below of=r4] (x4) {$X_4$}
			(x11) edge[blue, bend left=25] (x31)
			(x11) edge[blue] (r2)
			(x31) edge[blue, bend right=10] (r1)
			(x11) edge[blue] (r4)
			(r2) edge[blue] (r1)
			(r1) edge[red, <->, bend left] (r3)
			(r1) edge[red, <->, bend right=-5] (x4)
			(r3) edge[red, <->] (x4)
			(r1) edge[red, <->, bend right=10] (x2)
			(r2) edge[blue] (r3)
			(r3) edge[blue, bend left=30] (r1)
			(r4) edge[blue, bend left] (r2)
			(r1) edge[gray] (x1)
			(r2) edge[gray] (x2)
			(r3) edge[gray] (x3)
			(r4) edge[gray] (x4)
			(x31) edge[blue, bend left=9] (x2)
			(x11) edge[gray, bend right] (x1)
			(x31) edge[gray, bend left] (x3)
			
			node [below of=x2, xshift=0.75cm, yshift=0.7cm] {(b)}
			;
			\end{scope}
			\end{tikzpicture}
		}
	\end{center}
	\caption{(a) A DAG where the fixing operator must be performed on a set of vertices.
		(b) A latent projection of a subproblem used for identification of $p(R_4 | X_1^{(1)})$.
	}
	\label{fig:setFix}
\end{figure}
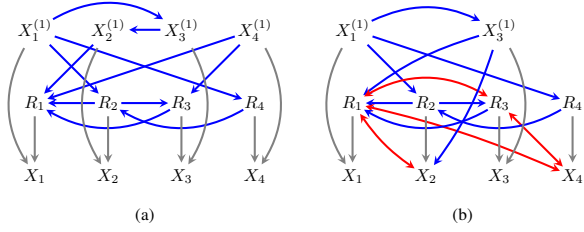

Consider the graph in Fig.~\ref{fig:setFix}(a). For $R_1$ and $R_3$, the fixing schedules are empty, and we immediately obtain their distributions as $p(R_1 | X^{(1)}_2, X^{(1)}_4, R_2, R_3) = p(R_1 | X_2, X_4, R_3, \bias{R_2, R_4})$ and $p(R_3 | X^{(1)}_4, R_2) = p(R_3 | X_4, \bias{R_4}, R_2)$.  For $R_2$, the partial order is $R_3 \prec R_1$ in a graph where we treat $X^{(1)}_2$ as a hidden variable $U_2$.
This yields $p(R_2 | X^{(1)}_1, R_4) = q_2(R_2 | X_1^{(1)}, R_4, \bias{R_1, R_3})$, where 
${\small q_2(X_1^{(1)}, X_2, X_3^{(1)}, X_4, R_2, \bias{R_4} | \bias{R_1, R_3})}$ is equal to
{\scriptsize
$
\frac{q_1(X_1, X_2, X_3^{(1)}, X_4, R_1, R_2, \bias{R_4} | \bias{R_3})}{q_1(\bias{R_1} | X_2, X_3, X_4, R_2, \bias{R_3, R_4})}, 
$
}
and 

{\scriptsize
\[
q_1( X_1, X_2, X_3^{(1)}, X_4, R_1, R_2, \bias{R_4} | \bias{R_3}) = 
\frac{p({\bf X},R_1,R_2,\bias{R_3, R_4})}{p(\bias{R_3} | R_2, X_4,\bias{R_4})}.
\]
}%
In order to obtain the propensity score for $R_4$ we must either render $X^{(1)}_1$ observable through fixing $R_1$ or perform valid fixing operations until we obtain a kernel in which $R_4$ is conditionally independent of $R_1$ given its parent $X_1^{(1)}$.
However, there exists no partial order on elements of ${\bf R}$. All partial orders on elements in $\bf R$ induce selection bias on variables higher in the order, preventing the identification of the required distribution for $R_4$.
For example, choosing a partial fixing order of $R_1 \prec R_3$, where we treat $X^{(1)}_2$ and $X^{(1)}_4$ as hidden variables results in selection bias on $R_3$ as soon as we fix $R_1$. Other partial orders fail similarly.
However, the following approach is possible in the graph in which we treat $X^{(1)}_2$ and $X^{(1)}_4$ as hidden variables.

$R_1$ and $R_3$ lie in the same district in the resulting latent projection ADMG, shown in Fig.~\ref{fig:setFix}(b).  Moreover, the set $\{ R_1, R_3 \}$ is closed under descendants in the district in Fig.~\ref{fig:setFix}(b).
As a result, $R_1$ and $R_3$ can essentially be viewed as a single vertex from the point of view of fixing.  Indeed we may  choose a partial order $\{R_1, R_3 \} \prec R_2$, where we fix $R_1$ and $R_3$ as a set.  
The fixing operation on the set is possible since $p( \bias{R_1, R_3} | \mb(R_1, R_3)) = p( \bias{R_1, R_3} | R_2, R_4, X_2, X^{(1)}_3, X_4)$ is a function of observed data law, $p({\bf X},{\bf R})$.  
Specifically, it is equal to 
{\small
$
p( \bias{R_3} | R_2, R_4, X_2, X_4) p( \bias{R_1} | R_2, R_4, X_2, X_3, X_4,  \bias{R_3}),
$
}
where the equality holds by d-separation ($R_3 \ci X^{(1)}_3 | R_2, R_4, X_2, X_4$). We then obtain 
${\small p(R_4 | X^{(1)}_1) = \frac{\sum_{X^{(1)}_3, X_4} q_2(X^{(1)}_1, X^{(1)}_3, X_4, R_4 | \bias{R_1, R_2, R_3})}
{\sum_{X^{(1)}_3, X_4, R_4} q_2(X^{(1)}_1, X^{(1)}_3, X_4, R_4 | \bias{R_1, R_2, R_3})},}$
where 
$	q_2(. | \bias{{\bf R}\setminus R_4}) = \tfrac{q_1(X_1^{(1)}, X_2, X_3^{(1)}, X_4, R_2, R_4 | \bias{R_1, R_3})}{q_1(R_2 | X^{(1)}_1, R_4, \bias{R_1, R_3})}, $ 
and $q_1(. | \bias{R_1, R_3}) = \frac{p({\bf X}, R_2, R_4, \bias{R_1, R_3})}{p(\bias{R_1, R_3} | R_2, R_4, X_2, X^{(1)}_3, X_4)}.$

Our final example demonstrates that in order to identify the target law, we may potentially need to fix variables outside ${\bf R}$, including variables in ${\bf X}^{({\bf 1})}$ that become observed after fixing or conditioning on some elements of ${\bf R}$.
Fig.~\ref{fig:UAI}(a) contains a generalization of the model considered in \cite{shpitser15missing}, where $O_3$ is fully observed.
In this model, distributions for $R_4$ and $R_1$ are identified immediately, while identification of $R_2$ requires a partial order $R_4 \prec X^{(1)}_4 \prec O_3 \prec R_1$ in the graph where we treat $X^{(1)}_1, X^{(1)}_2, X^{(1)}_4$ as latent variables (with the latent projection ADMG shown in Fig.~\ref{fig:UAI}(b)) until they are rendered observed by fixing the corresponding missingness indicators.  To illustrate fixing operations according to this order, the intermediate graphs that arise are shown in Fig.~\ref{fig:UAI}(c),(d),(e),(f).

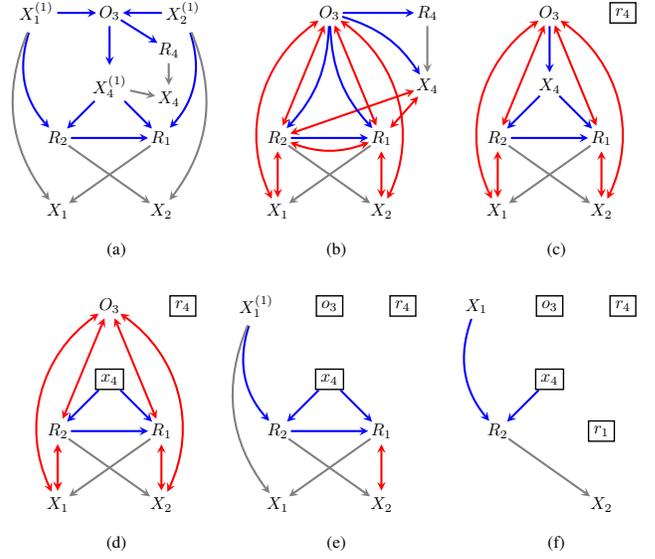
\begin{figure}[t]
\begin{center}
\scalebox{0.65}{
\begin{tikzpicture}[>=stealth, node distance=1.5cm]
    \tikzstyle{format} = [thick, circle, minimum size=1.0mm, inner sep=0pt]
	\tikzstyle{square} = [draw, thick, minimum size=1.0mm, inner sep=3pt]
	\begin{scope}
		\path[->, very thick]
			node[format] (x11) {$X^{(1)}_1$}
			node[format, right of=x11] (x31) {$O_3$}
			node[format, right of=x31] (x21) {$X^{(1)}_2$} 
			node[format, below of=x31] (x41) {$X^{(1)}_4$} 
			node[format, below right of=x41] (r1) {$R_1$} 
			node[format, below left of=x41] (r2) {$R_2$} 
			node[format, below right of=x31, xshift=0.15cm, yshift=0.3cm] (r4) {$R_4$} 
			node[format, below of=r4, yshift=0.5cm] (x4) {$X_4$}
			node[format, below of=r2] (x1) {$X_1$} 
			node[format, below of=r1] (x2) {$X_2$} 
			(x11) edge[blue] (x31)			
			(x21) edge[blue] (x31)	
			(x31) edge[blue] (x41)		
			(x41) edge[blue] (r2)			
			(x41) edge[blue] (r1)
			(r2)   edge[blue] (r1)		
			(x31) edge[blue] (r4)	
			(r1) edge[gray] (x1)
			(r2) edge[gray] (x2)	
			(r4) edge[gray] (x4)	
			(x41) edge[gray] (x4)	
			(x11) edge[blue, bend right=30] (r2)
			(x21) edge[blue, bend left=30] (r1)	
			(x11) edge[gray, bend right=30] (x1)	
			(x21) edge[gray, bend left=30] (x2)			
			node [below of=x1, xshift=1.2cm, yshift=0.7cm] {(a)}
			;
	\end{scope}
	\begin{scope}[xshift=4.5cm]
		\path[->, very thick]
			node[format] (x11) {}
			node[format, right of=x11] (x31) {$O_3$}
			node[format, below of=x31] (x41) {} 
			node[format, below right of=x41] (r1) {$R_1$} 
			node[format, below left of=x41] (r2) {$R_2$} 
			node[format, right of=x31, xshift=0.5cm] (r4) {$R_4$} 
			node[format, below of=r4] (x4) {$X_4$}
			node[format, below of=r2] (x1) {$X_1$} 
			node[format, below of=r1] (x2) {$X_2$} 
			(x31) edge[red, <->, bend left=40] (x2)
			(r1) edge[red, <->] (x2)
			(r2)   edge[blue] (r1)		
			(x31) edge[blue] (r4)	
			(r1) edge[gray] (x1)
			(r2) edge[gray] (x2)	
			(r4) edge[gray] (x4)	
			(x31) edge[red, <->, bend right=40] (x1)
			(x31) edge[red, <->] (r2)
			(x31) edge[red, <->] (r1)
			(r2) edge[red, <->] (x1)
			(r2) edge[red, bend right=20, <->] (r1)
			(r1) edge[red, <->] (x4)
			(r2) edge[red, <->] (x4)
			(x31) edge[blue, bend left=20] (r2)
			(x31) edge[blue, bend right=20] (r1)
			(x31) edge[blue, bend left=20] (x4)
			node [below of=x1, xshift=1.2cm, yshift=0.7cm] {(b)}	
			;
	\end{scope}
	\begin{scope}[xshift=9cm]
		\path[->, very thick]
			node[format] (x11) {}
			node[format, right of=x11] (x31) {$O_3$}
			node[format, below of=x31] (x41) {$X_4$} 
			node[format, below right of=x41] (r1) {$R_1$} 
			node[format, below left of=x41] (r2) {$R_2$} 
			node[square, right of=x31] (r4) {$r_4$} 
			node[format, below of=r4] (x4) {}
			node[format, below of=r2] (x1) {$X_1$} 
			node[format, below of=r1] (x2) {$X_2$} 
			(x31) edge[red, <->] (r1)
			(x31) edge[red, <->, bend left=40] (x2)
			(r1) edge[red, <->] (x2)
			(x31) edge[blue] (x41)		
			(x41) edge[blue] (r2)			
			(x41) edge[blue] (r1)
			(r2)   edge[blue] (r1)
			(r2) edge[red, <->] (x1)
			(r1) edge[gray] (x1)
			(r2) edge[gray] (x2)	
			(x31) edge[red, <->, bend right=40] (x1)
			(x31) edge[red, <->] (r2)
			node [below of=x1, xshift=1.2cm, yshift=0.7cm] {(c)}
			;
	\end{scope}
	\begin{scope}[yshift=-6cm]
		\path[->, very thick]
			node[format] (x11) {}
			node[format, right of=x11] (x31) {$O_3$}
			node[square, below of=x31] (x41) {$x_4$} 
			node[format, below right of=x41] (r1) {$R_1$} 
			node[format, below left of=x41] (r2) {$R_2$} 
			node[square, right of=x31] (r4) {$r_4$} 
			node[format, below of=r4] (x4) {}
			node[format, below of=r2] (x1) {$X_1$} 
			node[format, below of=r1] (x2) {$X_2$} 
			(x31) edge[red, <->] (r1)
			(x31) edge[red, <->, bend left=40] (x2)
			(r1) edge[red, <->] (x2)
			(x41) edge[blue] (r2)			
			(x41) edge[blue] (r1)
			(r2)   edge[blue] (r1)	
			(r2) edge[red, <->] (x1)	
			(r1) edge[gray] (x1)
			(r2) edge[gray] (x2)	
			(x31) edge[red, <->, bend right=40] (x1)
			(x31) edge[red, <->] (r2)
			(r2) edge[red, <->] (x1)
			node [below of=x1, xshift=1.2cm, yshift=0.7cm] {(d)}
			;
	\end{scope}
	\begin{scope}[xshift=4.5cm, yshift=-6cm]
		\path[->, very thick]
			node[format] (x11) {$X^{(1)}_1$}
			node[square, right of=x11] (x31) {$o_3$}
			node[square, below of=x31] (x41) {$x_4$} 
			node[format, below right of=x41] (r1) {$R_1$} 
			node[format, below left of=x41] (r2) {$R_2$} 
			node[square, right of=x31] (r4) {$r_4$} 
			node[format, below of=r4] (x4) {}
			node[format, below of=r2] (x1) {$X_1$} 
			node[format, below of=r1] (x2) {$X_2$} 
			(r1) edge[red, <->] (x2)
			(x41) edge[blue] (r2)			
			(x41) edge[blue] (r1)
			(r2)   edge[blue] (r1)	
			(r1) edge[gray] (x1)
			(r2) edge[gray] (x2)	
			(x11) edge[blue, bend right=30] (r2)
			(x11) edge[gray, bend right=30] (x1)
			node [below of=x1, xshift=1.2cm, yshift=0.7cm] {(e)}
			;
	\end{scope}
	\begin{scope}[xshift=9cm, yshift=-6cm]
		\path[->, very thick]
			node[format] (x11) {$X_1$}
			node[square, right of=x11] (x31) {$o_3$}
			node[square, below of=x31] (x41) {$x_4$} 
			node[square, below right of=x41] (r1) {$r_1$} 
			node[format, below left of=x41] (r2) {$R_2$} 
			node[square, right of=x31] (r4) {$r_4$} 
			node[format, below of=r4] (x4) {}
			node[format, below of=r2] (x1) {} 
			node[format, below of=r1] (x2) {$X_2$} 
			(x41) edge[blue] (r2)			
			(r2) edge[gray] (x2)	
			(x11) edge[blue, bend right=30] (r2)
			node [below of=x1, xshift=1.2cm, yshift=0.7cm] {(f)}
			;
	\end{scope}
\end{tikzpicture}
}
\end{center}
\caption{A DAG where variables besides $R$s are required to be fixed. }
\label{fig:UAI}
\end{figure}

\section{A NEW IDENTIFICATION ALGORITHM}
\label{sec:alg}

In order to identify the target law in examples discussed in the previous section, we had to consider situations where some variables were viewed as hidden, and marginalized out, and others were conditioned on, introducing selection bias.  In addition, fixing operations were performed according to a partial, rather than a total, order as was the case in causal inference problems.  Finally, we sometimes fixed sets of variables jointly, rather than individual variables.  We now introduce relevant definitions that allow us to formulate 
a general identification algorithm that takes advantage of all these techniques. 

Let ${\bf V}$ be a set of random variables (and corresponding vertices) consisting of observed variables ${\bf O},{\bf R},{\bf X}$, missing variables ${\bf X}^{({\bf 1})}$, and selected variables ${\bf S}$. Let ${\bf W}$ be a set of fixed observed variables.
The following definitions apply to a 
{latent projection ${\cal G}({\bf V} \setminus {\bf X}^{({\bf 1})}_{\bf U},{\bf W})$, for some ${\bf X}^{({\bf 1})}_{\bf U} \subseteq {\bf X}^{({\bf 1})}$},
 and a corresponding kernel $q({\bf V} \setminus {\bf X}^{({\bf 1})}_{\bf U} | {\bf W}) \equiv \sum_{{\bf X}^{({\bf 1})}_{\bf U}} q({\bf V} | {\bf W})$. Graph
 {${\cal G}$ can be viewed as a latent variable CADMG for $q$ where ${\bf X}^{({\bf 1})}_{\bf U}$ are latent.} 
Such CADMGs represent intermediate subproblems in our identification algorithm.

For ${\bf Z} \subseteq {\bf D}_{\bf Z} \in {\cal D}({\cal G})$,
let ${\bf R}_{\bf Z} = \{ R_j | X_j^{(1)} \in {\bf Z} \cup \mb_{\cal G}({\bf Z}), R_j \not\in {\bf Z} \}$, and
$\mb_{\cal G}({\bf Z}) \equiv ({\bf D}_{\bf Z} \cup \pa_{\cal G}({\bf D}_{\bf Z})) \setminus {\bf Z}$.
We say ${\bf Z}$ is \emph{fixable} in ${\cal G}({\bf V} \setminus {\bf X}^{({\bf 1})}_{\bf U},{\bf W})$ if
\begin{enumerate}[label=(\roman*)]
\item 
$\de_{\cal G}({\bf Z}) \cap {\bf D}_{\bf Z} \subseteq {\bf Z}$,
\item ${\bf S} \cap {\bf Z} = \emptyset$,
\item ${\bf Z} \ci ({\bf S} \cup {\bf R}_{\bf Z}) \setminus \mb_{\cal G}({\bf Z}) | \mb_{\cal G}({\bf Z})$.
\end{enumerate}

In words, these conditions apply to some $\bf Z$ that is a subset of its own district (which is trivial when the set ${\bf Z}$ is a singleton). 
The conditions, in the listed order, require that ${\bf Z}$ is closed under descendants in the district, should not contain any selected variables, and should be independent of both selected variables $\bf S$ and the missingness indicators $\bf R_Z$ of the corresponding counterfactal parents given the Markov blanket of $\bf Z$, respectively. Consider the graph in Fig.~\ref{fig:setFix}(b) where $\bf S = \emptyset$ and let ${\bf Z} = \{R_1, R_3\}$. ${\bf Z}$ is fixable since ${\bf Z} \subseteq {\bf D}_{\bf Z} = \{R_1, R_3, X_2, X_4 \}, \de_{\cal G}({\bf Z}) = \{R_1, R_3, X_1, X_3\} \cap  {\bf D}_{\bf Z} = \{R_1, R_3\}$ is closed, and both ${\bf S}$ and ${\bf R_Z}$ are empty sets.  

A set $\widetilde{\bf Z}$ spanning multiple elements in ${\cal D}({\cal G})$ is said to be fixable if it can be partitioned into a set ${\cal Z}$ of elements ${\bf Z}$, such that each ${\bf Z}$ is a subset of a single district in ${\cal D}({\cal G})$ and is fixable. 

Given an ordering $\prec$ on vertices ${\bf V}\cup{\bf W}$ topological in ${\cal G}$ and $\widetilde{\bf Z}$ fixable in ${\cal G}$,
define $\phi_{\widetilde{\bf Z}}(q; {\cal G})$ as
{\scriptsize
\begin{align}
\!\!\!\!\!\!
\frac{
q({\bf V} \setminus ({\bf X}^{({\bf 1})}_{\bf U} \cup {\bf R}_{\bf Z}), {\bf R_Z=1} | {\bf W})
}{
\prod\limits_{{\bf Z} \in {\cal Z}}
\prod\limits_{Z \in {\bf Z}} q(Z | \mb_{\cal G}(Z; \an_{\cal G}({\bf D}_{\bf Z}) \cap \{\preceq Z\})), {\bf R_Z}) \vert_{({\bf R} \cap {\bf Z}) \cup {\bf R}_{\bf Z} = {\bf 1}}
},
\label{eqn:fix-set}
\end{align}
}%
where $\mb_\G(V; {\bf S})\equiv \mb_{\G_{\bf S}}(V)$ and $\{ \preceq Z \} $ is the set of all elements earlier than $Z$ in the order $\prec$ (this includes $Z$ itself).

Given a set ${\bf Z} \subseteq {\bf R}\cup{\bf O}\cup{\bf X}^{({\bf 1})}$, and an equivalence relation $\sim$, let $\quotient{{\bf Z}}{\sim}$ be the partition of ${\bf Z}$ into equivalence classes according to $\sim$.
Define a \emph{fixing schedule} for $\quotient{{\bf Z}}{\sim}$ to be a partial order $\lhd$ on $\quotient{{\bf Z}}{\sim}$.
For each ${\bf Z} \in \quotient{{\bf Z}}{\sim}$, define $\{\unlhd \widetilde{\bf Z}\}$ to be the set of elements in $\quotient{{\bf Z}}{\sim}$ earlier than $\widetilde{\bf Z}$ in the order $\lhd$, and $\{ \lhd \widetilde{\bf Z} \} \equiv \{ \unlhd \widetilde{\bf Z}\} \setminus \widetilde{\bf Z}$.
Define $\unlhd_{\widetilde{\bf Z}}$ and $\lhd_{\widetilde{\bf Z}}$ to be restrictions of $\lhd$ to $\{ \unlhd \widetilde{\bf Z} \}$ and $\{ \lhd \widetilde{\bf Z} \}$, respectively. Both restrictions, $\unlhd_{\widetilde{\bf Z}}$ and $\lhd_{\widetilde{\bf Z}}$, are also partial orders.

We inductively define a \emph{valid} fixing schedule (a schedule where fixing operations can be successfully implemented), along with the fixing operator on valid schedules.  The fixing operator will implement fixing as in (\ref{eqn:fix-set}) on $\widetilde{\bf Z}$ within an intermediate problem represented by a CADMG where some ${\bf X}^{({\bf 1})}_{\widetilde{\bf Z}} \subseteq {\bf X}^{({\bf 1})}$ will become observed after fixing $\widetilde{{\bf Z}}$, with ${\bf X^{(1)} \setminus X^{(1)}_{\widetilde{{\bf Z}}}}$
treated as latent variables, and a kernel associated with this CADMG defined on the observed subset of variables. We also define
${\bf X}^{({\bf 1})}_{\{ \unlhd \widetilde{\bf Z} \}} \equiv \bigcup_{{\bf Z} \in \{ \unlhd \widetilde{\bf Z} \}} {\bf X}^{({\bf 1})}_{\bf Z}$.



We say $\lhd_{\widetilde{\bf Z}}$ is valid for $\{ \lhd {\widetilde{\bf Z}} \}$ in ${\cal G}$ if for every $\lhd$-largest element $\widetilde{\bf Y}$ of $\{ \lhd {\widetilde{\bf Z}} \}$, $\unlhd_{\widetilde{\bf Y}}$ is valid for $\{ \unlhd \widetilde{\bf Y} \}$.
If $\lhd_{\widetilde{\bf Z}}$ is valid for $\{ \lhd {\widetilde{\bf Z}} \}$,
we define $\phi_{\lhd_{{\widetilde{\bf Z}}}}({\cal G})$
to be a new CADMG ${\cal G}({\bf V} \setminus \bigcup_{{\bf Z} \in \{ \lhd {\widetilde{\bf Z}} \}} {\bf Z}, {\bf W} \cup \bigcup_{{\bf Z} \in \{ \lhd {\widetilde{\bf Z}} \}} {\bf Z})$ obtained from ${\cal G}({\bf V},{\bf W})$ by:
\begin{itemize}
\item Removing all edges with arrowheads into $\bigcup_{{\bf Z} \in \{ \lhd {\widetilde{\bf Z}} \}} {\bf Z}$,
\item Marking any $\{ X_j^{(1)} | X_j^{(1)} \in {\bf Z} \cup \mb_{{\phi}_{\lhd_{\bf Z}}({\cal G})}({\bf Z}), {\bf Z} \in \{ \lhd {\widetilde{\bf Z}} \} \}$ as observed,
\item Marking any $\{ {\bf R_Z} \cap {\bf V} | {\bf Z} \in \{ \lhd {\widetilde{\bf Z}} \} \} \setminus \bigcup_{{\bf Z} \in \{ \lhd {\widetilde{\bf Z}} \}} {\bf Z}$
as selected to value $1$, where ${\bf R_Z}$ is defined with respect to $\phi_{\lhd_{{\bf Z}}}({\cal G})$
\item Treating elements of ${\bf X^{(1)} \setminus X^{(1)}_{\widetilde{{\bf Z}}}}$ as hidden variables.
\end{itemize}
%
%
%
We say $\unlhd_{\widetilde{\bf Z}}$ is valid for $\{ \unlhd {\widetilde{\bf Z}} \}$ if $\lhd_{\widetilde{\bf Z}}$ is valid for $\{ \lhd {\widetilde{\bf Z}} \}$ and
$\widetilde{\bf Z}$ is fixable in $\phi_{\lhd_{\widetilde{\bf Z}}}({\cal G})$.  If $\unlhd_{\widetilde{\bf Z}}$ is valid, we define 
{\small
\begin{align}
\phi_{\unlhd_{\widetilde{\bf Z}}}(q; {\cal G}) \equiv \phi_{\widetilde{\bf Z}}\left(
\phi_{\lhd_{\widetilde{\bf Z}}}(q; {\cal G})
;
\phi_{\lhd_{\widetilde{\bf Z}}}({\cal G})\right),
\label{eqn:ipw}
\end{align}
}
where
$\phi_{\lhd_{\widetilde{\bf Z}}}(q; {\cal G})\equiv \frac{
q({\bf V} | {\bf W})
}{
\prod_{\widetilde{\bf Y} \in \{\lhd \widetilde{\bf Z}\}} q_{\widetilde{\bf Y}}
}$, and $q_{\widetilde{\bf Y}}$ are defined inductively as the denominator of (\ref{eqn:fix-set}) for $\widetilde{\bf Y}$,
$\phi_{\lhd_{\widetilde{\bf Y}}}({\cal G})$ and $\phi_{\lhd_{\widetilde{\bf Y}}}(q; {\cal G})$.

We have the following claims.
\begin{prop}
Given a DAG ${\cal G}({\bf X}^{({\bf 1})}, {\bf R}, {\bf O}, {\bf X})$, the distribution $p(R_i | \pa_{\cal G}(R_i)) \vert_{\pa_{\cal G}(R_i) \cap {\bf R} = {\bf 1}}$ is identifiable from $p({\bf R},{\bf O},{\bf X})$ if there exists
\begin{enumerate}[label=(\roman*)]
\item
${\bf Z} \subseteq {\bf X}^{({\bf 1})} \cup {\bf R} \cup {\bf O}$, 

\item an equivalence relation $\sim$ on ${\bf Z}$ such that $\{ R_i \} \in \quotient{\bf Z}{\sim}$, 

\item a set of elements ${\bf X}^{({\bf 1})}_{\widetilde{\bf Z}}$ such that ${\bf X}^{({\bf 1})}_{\{\lhd \widetilde{\bf Z} \}} \subseteq {\bf X}^{({\bf 1})}_{\widetilde{\bf Z}} \subseteq {\bf X}^{({\bf 1})}$ for each $\widetilde{\bf Z} \in \quotient{\bf Z}{\sim}$,

\item ${\bf X}^{({\bf 1})} \cap \pa_{\cal G}(R_i) \subseteq ({\bf Z} \setminus \{ R_i \}) \cup {\bf X}^{({\bf 1})}_{\{ R_i \}}$,

\item and a valid fixing schedule $\lhd$ for $\quotient{\bf Z}{\sim}$ in ${\cal G}$ such that for each $\widetilde{\bf Z} \in \quotient{\bf Z}{\sim}$, $\widetilde{\bf Z} \lhd \{ R_i \}$.
\end{enumerate}
Moreover, $p(R_i | \pa_{\cal G}(R_i)) \vert_{\pa_{\cal G}(R_i) \cap {\bf R} = {\bf 1}}$ is equal to $q_{\{ R_i \}}$, defined inductively as the denominator of (\ref{eqn:fix-set}) for $\{ R_i \}$, $\phi_{\lhd_{\{ R_i \}}}({\cal G})$ and $\phi_{\lhd_{\{ R_i \}}}(p; {\cal G})$, and evaluated at
$\pa_{\cal G}(R_i) \cap {\bf R} = {\bf 1}$.
\label{prop:r}
\end{prop}

Proposition \ref{prop:r} implies that $p(R_i | \pa_{\G}(R_i))$ is identified if we can find a set of variables that can be fixed according to a partial order (possibly through set fixing) within subproblems where certain variables are hidden. At the end of the fixing schedule, we require that $R_i$ itself is fixable given its Markov blanket in the original DAG.  We encourage the reader to view the example provided in Appendix B, for a demonstration of valid fixing schedules that may be chosen by Proposition \ref{prop:r}.

\begin{cor}
Given a DAG ${\cal G}({\bf X}^{({\bf 1})}, {\bf R}, {\bf O}, {\bf X})$, the target law $p({\bf X}^{({\bf 1})}, {\bf O})$ is identified if
$p(R_i | \pa_{\cal G}(R_i))$ is identified via Proposition \ref{prop:r} for every $R_i \in {\bf R}$.
\end{cor}
\begin{proof}
Follows by Proposition \ref{prop:r} and (\ref{eqn:chain}).
\end{proof}

In addition, in special classes of models, the full law, rather than just the target law is identified.
\begin{prop}
Given a DAG ${\cal G}({\bf X}^{({\bf 1})}, {\bf R}, {\bf O}, {\bf X})$, the full law $p({\bf R},{\bf X}^{({\bf 1})}, {\bf O})$
is identifiable from $p({\bf R},{\bf O},{\bf X})$ if for every $R_i \in {\bf R}$, all conditions in Proposition \ref{prop:r} (i-v) are met, and also for each ${\bf \widetilde{Z}} \in \quotient{\bf Z}{\sim}$, ${\bf X}^{({\bf 1})}_{\bf \widetilde{Z}}$ does not contain any elements in $\{ X^{(1)}_j | R_j \in \pa_{\cal G}(R_i) \}$. 
Moreover, $p(R_i | \pa_{\cal G}(R_i))$ is equal to $q_{\{ R_i \}}$, defined inductively as the denominator of (\ref{eqn:fix-set}) for $\{ R_i \}$,
$\phi_{\lhd_{\{ R_i \}}}({\cal G})$ and $\phi_{\lhd_{\{ R_i \}}}(p; {\cal G})$, and
{\small
\begin{align*}
p({\bf R},{\bf X}^{({\bf 1})}, {\bf O})
=
\Big( \prod_{R_i \in {\bf R}} q_{R_i} \Big) \times
\frac{
p({\bf R}={\bf 1},{\bf O},{\bf X})
}{
\left( \prod_{R_i \in {\bf R}} q_{R_i} \right) \vert_{{\bf R} = {\bf 1}}
}
\end{align*}
}
\label{prop:r-full}
\end{prop}
\begin{proof}
	Under conditions (i-v) in Proposition \ref{prop:r}, we are guaranteed to identify the target law and obtain $p(R_i | \pa_{\G}(R_i))$ where some $R_j \in \pa_{\G}(R_i)$ may be evaluated at $R_j = 1$. Under the additional restriction stated above, all $R_j \in \pa_{\G}(R_i)$ can be evaluated at  all levels.
\end{proof}

Proposition \ref{prop:r-full} always fails if a special collider structure $X_j^{(1)} \to R_i \gets R_j$, which we call \emph{the colluder}, exists in ${\cal G}$. The following Lemma implies that colluders always imply the full law is not identified.


\begin{lemma}
\label{lem:colluder}
	In a DAG  ${\cal G}({\bf X}^{({\bf 1})}, {\bf R}, {\bf O}, {\bf X})$, if there exists $R_i, R_j \in {\bf R}$ such that $\{R_j, X^{(1)}_j \} \in \pa_{\cal G}(R_i)$, then $p(R_i | \pa_{\cal G}(R_i)) \vert_{R_j = 0} $ is not identified. Hence, the full law $p({\bf X}^{({\bf 1})}, {\bf R})$ is not identified. 
\end{lemma}

\begin{proof}
	Follows by providing two different full laws that agree on the observed law on a DAG with 2 counterfactual random variables (Appendix C). This result holds for an arbitrary DAG representing a missing data model that contains the colluder structure mentioned above. 
\end{proof}

Propositions \ref{prop:r} and \ref{prop:r-full} do not address a computationally efficient search procedure for 
a valid fixing schedule $\lhd$ that permit identification of $p(R_i | \pa_{\cal G}(R_i))$ for a particular $R_i \in {\bf R}$.  
Nevertheless, the following Lemma shows how to easily obtain identification of the target law in a restricted class of missing data DAGs.

\begin{lemma}
\label{lem:property}
Consider a DAG ${\cal G}({\bf X}^{({\bf 1})}, {\bf R}, {\bf O}, {\bf X})$ such that for every $R_i \in {\bf R}$,
$\{ R_j | X_j^{(1)} \in \pa_{\cal G}(R_i) \} \cap \an_{\cal G}(R_i) = \emptyset$.  Then for every $R_i \in {\bf R}$, a fixing schedule $\lhd$
for $\{ \{ R_j \} | R_j \in {\cal G}_{{\bf R} \cap \de_{\cal G}(R_i)} \}$ given by the partial order induced by the ancestrality relation on
${\cal G}_{{\bf R} \cap \de_{\cal G}(R_i)}$ is valid in ${\cal G}({\bf X}^{({\bf 1})}, {\bf R}, {\bf O}, {\bf X})$, by taking each $\bf X^{(1)}_{\widetilde{{\bf Z}}} =  \bigcup_{{\bf Z} \in \{ \unlhd \widetilde{\bf Z} \}} {\bf X}^{({\bf 1})}_{\bf Z}$,
 for every $\widetilde{\bf Z} \in \{ \unlhd \{ R_i \} \}$. Thus the target law is identified.
\end{lemma}
\section{DISCUSSION AND CONCLUSION}
\label{sec:conclusion}

In this paper we addressed the significant gap present in identification theory for missing data models representable as DAGs. We showed, by examples, that straightforward application of identification machinery in causal inference with hidden variables do not suffice for identification in missing data, and discussed the generalizations required to make it suitable for this task. These generalizations included fixing (possibly sets of) variables on a partial order and avoiding selection bias by introducing hidden variables into the problem though they were not present in the initial problem statement. Proposition \ref{prop:r} gives a characterization of how to utilize these generalized procedures to obtain identification of the target law, while Proposition \ref{prop:r-full} gives a similar characterization for the full law. While neither of these propositions alluded to a computationally efficient algorithm to obtain identification in general, Lemma \ref{lem:property} provides such a procedure for a special class of missing data models where the partial order of fixing operations required for each $R$ is easy to determine. Providing a computationally efficient search procedure for identification in all DAG models of missing data, and questions regarding the completeness of our proposed algorithm are left for future work. 

\subsection*{Acknowledgements}
This project is sponsored in part by the National Institutes of Health grant R01 AI127271-01 A1 and the Office of Naval Research grant N00014-18-1-2760.

\clearpage
\bibliographystyle{plain}
\bibliography{references}

\clearpage

\section{APPENDIX}
\label{sec:appendix}

\subsection*{A. Proofs}
\begin{proa}{\ref{prop:r}}
	Given a DAG ${\cal G}({\bf X}^{({\bf 1})}, {\bf R}, {\bf O}, {\bf X})$, the distribution $p(R_i | \pa_{\cal G}(R_i)) \vert_{\pa_{\cal G}(R_i) \cap {\bf R} = {\bf 1}}$ is identifiable from $p({\bf R},{\bf O},{\bf X})$ if there exists
	\begin{enumerate}[label=(\roman*)]
		\item
		${\bf Z} \subseteq {\bf X}^{({\bf 1})} \cup {\bf R} \cup {\bf O}$, 
		
		\item an equivalence relation $\sim$ on ${\bf Z}$ such that $\{ R_i \} \in \quotient{\bf Z}{\sim}$, 
		
		\item a set of elements ${\bf X}^{({\bf 1})}_{\widetilde{\bf Z}}$ such that  ${\bf X}^{({\bf 1})}_{\{\lhd \widetilde{\bf Z} \}} \subseteq {\bf X}^{({\bf 1})}_{\widetilde{\bf Z}} \subseteq {\bf X}^{({\bf 1})}$ for each $\widetilde{\bf Z} \in \quotient{\bf Z}{\sim}$,
		
		\item ${\bf X}^{({\bf 1})} \cap \pa_{\cal G}(R_i) \subseteq ({\bf Z} \setminus \{ R_i \}) \cup {\bf X}^{({\bf 1})}_{\{ R_i \}}$,
		
		\item and a valid fixing schedule $\lhd$ for $\quotient{\bf Z}{\sim}$ in ${\cal G}$ such that for each $\widetilde{\bf Z} \in \quotient{\bf Z}{\sim}$, $\widetilde{\bf Z} \lhd \{ R_i \}$.
	\end{enumerate}
	Moreover, $p(R_i | \pa_{\cal G}(R_i)) \vert_{\pa_{\cal G}(R_i) \cap {\bf R} = {\bf 1}}$ is equal to $q_{\{ R_i \}}$, defined inductively as the denominator of (\ref{eqn:fix-set}) for $\{ R_i \}$, $\phi_{\lhd_{\{ R_i \}}}({\cal G})$ and $\phi_{\lhd_{\{ R_i \}}}(p; {\cal G})$, and evaluated at
	$\pa_{\cal G}(R_i) \cap {\bf R} = {\bf 1}$.
\end{proa}

\begin{proof}
	
	We first outline the essential argument made in this proof. 
	We will reformulate the process of fixing according to a partial order in a missing data problem as a problem of ordinary fixing based on a total order in a causal inference problem where, previously missing variables are in fact observed.
	If we are able to show this, we can invoke results from \cite{richardson17nested}, that guarantee that we obtain the desired conditional for each $R_i$.
	
	Consider $\widetilde{\bf Z} \in \quotient{\bf Z}{\sim}$, and define
	${\bf X}^{({\bf 1})}_{\{ \unlhd \widetilde{\bf Z} \}} \equiv \bigcup_{{\bf Z} \in \{ \unlhd \widetilde{\bf Z} \}} {\bf X}^{({\bf 1})}_{\bf Z}$,
	and 
	${\bf R}_{\{ \unlhd \widetilde{\bf Z} \}} \equiv \{ R_k | X^{(1)}_k \in {\bf X}^{({\bf 1})}_{\{ \unlhd \widetilde{\bf Z} \}} \}$, 
	and similarly for
	${\bf X}^{({\bf 1})}_{\{ \lhd \widetilde{\bf Z} \}}$ and ${\bf R}_{\{ \lhd \widetilde{\bf Z} \}}$.
	
	We first note that any total ordering $\prec$ on $\{ \lhd \widetilde{\bf Z} \}$ consistent with $\lhd$
	yields a valid fixing sequence on sets in $\{ \lhd \widetilde{\bf Z} \}$ in ${\cal G}({\bf R},{\bf O},{\bf X}^{({\bf 1})}, {\bf X}))$,
	where ${\bf X}^{({\bf 1})}_{\{ \lhd \widetilde{\bf Z} \}},{\bf R},{\bf O},{\bf X}$ are observed. The total ordering $\prec$ can be refined to operate on single variables where each set ${\bf \widetilde{Z}}$ is fixed as singletons following a topological total order where variables with no children in $\widetilde{{\bf Z}}$ would be fixed first.
	Such a total order is also valid and follows from the validity of $\lhd$ and the fact that at each step of the fixing operation in the total order, the Markov blanket of each $Z$ contains only observed variables; hence no selection bias is induced on any singleton variables $\{ \succ \widetilde{{\bf Z}} \}$. 

	We now show, by induction on the structure of the partial order $\lhd$, that for a particular $\widetilde{\bf Z} \in \quotient{\bf Z}{\sim}$,
	$q_{\widetilde{\bf Z}}$ is equal to
	{\small
		\begin{align}
		\!\!\!\!
		\prod\limits_{{\bf Z} \in {\cal Z}}
		\prod\limits_{Z \in {\bf Z}} \tilde{q}(Z | \mb_{\tilde{\cal G}}(Z; \an_{\tilde{\cal G}}({\bf D}_{\bf Z}) \cap \prec_{\tilde{\G}}\{Z \}, {\bf R}_{\bf Z}) \vert_{({\bf R} \cap {\bf Z}) \cup {\bf R}_{\bf Z} = {\bf 1}},
		\end{align}}%
	obtained from a kernel 
	\[
	\tilde{q} \equiv \phi_{\{ \lhd \widetilde{\bf Z} \}}(p({\bf R}, {\bf O}, {\bf X}^{({\bf 1})}_{\{ \lhd \widetilde{\bf Z} \}}, {\bf X}); {\cal G}),
	\]
	and CADMG 
	\[
	\tilde{\cal G} \equiv \phi_{\{ \lhd \widetilde{\bf Z} \}}({\cal G}({\bf R},{\bf O},{\bf X}^{({\bf 1})}_{\{ \lhd \widetilde{\bf Z} \}}, {\bf X})),
	\] 
	where ${\bf X}^{({\bf 1})}_{\{ \lhd \widetilde{\bf Z} \}},{\bf R},{\bf O},{\bf X}$ are observed.
	
	For any $\lhd$-smallest $\widetilde{\bf Z}$, $\widetilde{{\bf Z}}$ is independent of ${\bf R}_{\{ \unlhd \widetilde{\bf Z} \}}$ given its Markov blanket; therefore treating  ${\bf X}^{({\bf 1})}_{\{ \unlhd \widetilde{\bf Z} \}}$ as observed results in the same kernel as $q_{\widetilde{{\bf Z}}}$.
	
	We now show that the above is also true for any $\widetilde{\bf Z} \in \quotient{{\bf Z}}{\sim}$. Assume the inductive hypothesis holds for all $\widetilde{\bf Y} \in \{ \lhd \widetilde{\bf Z} \}$.
	Since $\lhd$ is valid, we obtain $q_{\widetilde{\bf Z}}$ by applying
	{\small
		\begin{align}
		&\phi_{\unlhd_{\widetilde{\bf Z}}}(q; {\cal G}) \equiv \nonumber \\ 
		& \hspace{0.5cm} \phi_{\widetilde{\bf Z}}\Big(
		\frac{
			p({\bf O},{\bf X},{\bf R} \setminus {\bf R}_{\{ \lhd \widetilde{\bf Z} \}}, {\bf R}_{\{ \lhd \widetilde{\bf Z} \}} = {\bf 1})
		}{
			\prod_{\widetilde{\bf Y} \in \{\lhd \widetilde{\bf Z}\}} q_{\widetilde{\bf Y}}
		}
		;
		\phi_{\lhd_{\widetilde{\bf Z}}}({\cal G})\Big),
		\label{eqn:ipw}
		\end{align}
	}%
	where
	$q_{\widetilde{\bf Y}}$ are defined by the inductive hypothesis, and
	$\phi_{\widetilde{\bf Z}}$ is defined via
	{\small
		\begin{align}
		\frac{
			q({\bf V} \setminus (({\bf X}^{({\bf 1})} \setminus {\bf X}^{({\bf 1})}_{\{ \lhd \widetilde{\bf Z} \}}) \cup {\bf R}_{\bf Z}), {\bf R}_{\bf Z} = {\bf 1} | {\bf W})
		}{
			\prod\limits_{{\bf Z} \in {\cal Z}}
			\prod\limits_{Z \in {\bf Z}} q(Z | \mb_{\tilde{\G}}(Z; \an_{\tilde{\G}}({\bf D}_{\bf Z}) \cap \prec_{\tilde{\G}}(Z)), {\bf R}_{\bf Z}) \vert_{({\bf R} \cap {\bf Z}) \cup {\bf R}_{\bf Z} = {\bf 1}}
		},
		\label{eqn:fix-set-miss}
		\end{align}
	}%
	where 
	{\small
		\[
		q({\bf V} \setminus ({\bf X}^{({\bf 1})} \setminus {\bf X}^{({\bf 1})}_{\{ \lhd \widetilde{\bf Z} \}}) | {\bf W})
		\equiv \frac{
			p({\bf O},{\bf X},{\bf R} \setminus {\bf R}_{\{ \lhd \widetilde{\bf Z} \}}, {\bf R}_{\{ \lhd \widetilde{\bf Z} \}} = {\bf 1})
		}{
			\prod_{\widetilde{\bf Y} \in \{\lhd \widetilde{\bf Z}\}} q_{\widetilde{\bf Y}}
		}.
		\]
	}%
	
	Consider the equivalent functional in the model where we observe ${\bf X^{(1)}_{\{ \lhd \widetilde{Z} \}}}$
	{\small
		\begin{align}
		\frac{
			{q^\dagger}({\bf V} \setminus (({\bf X}^{({\bf 1})} \setminus {\bf X}^{({\bf 1})}_{\{ \lhd \widetilde{\bf Z} \}}) \cup {\bf R}_{\bf Z}), {\bf R}_{\bf Z} = {\bf 1} | {\bf W})
		}{
			\prod\limits_{{\bf Z} \in {\cal Z}}
			\prod\limits_{Z \in {\bf Z}} {q^\dagger}(Z | \mb_{\tilde{\cal G}}(Z; \an_{\tilde{\cal G}}({\bf D}_{\bf Z}) \cap \prec_{\tilde{\G}}(Z)), {\bf R}_{\bf Z}) \vert_{({\bf R} \cap {\bf Z}) \cup {\bf R}_{\bf Z} = {\bf 1}}
		},
		\label{eqn:fix-set-obs}
		\end{align}
	}%
	where 
	\begin{align*}
	{q^\dagger}({\bf V} \setminus & ({\bf X}^{({\bf 1})} \setminus {\bf X}^{({\bf 1})}_{\{ \lhd \widetilde{\bf Z} \}}) | {\bf W})
	\equiv \\
	& \hspace{0.7cm}	 \frac{
		p({\bf O},{\bf X},{\bf X}^{({\bf 1})}_{\{ \lhd \widetilde{\bf Z} \}}, {\bf R} \setminus \widetilde{\bf R}_{\{ \lhd \widetilde{\bf Z} \}}, \widetilde{\bf R}_{\{ \lhd \widetilde{\bf Z} \}} = {\bf 1})
	}{
		\prod_{\widetilde{\bf Y} \in \{\lhd \widetilde{\bf Z}\}} q_{\widetilde{\bf Y}}
	}, 
	\end{align*}%
	and $\widetilde{\bf R}_{\{ \lhd \widetilde{\bf Z} \}}$ is defined as the subset of ${\bf R}_{\{ \lhd \widetilde{\bf Z} \}}$ that is fixed in $\{ \lhd \widetilde{\bf Z} \}$.
	
	The only difference between (\ref{eqn:fix-set-miss}) and (\ref{eqn:fix-set-obs}) for the purposes of the denominator is the variables in
	${\bf R}_{\{ \lhd \widetilde{\bf Z} \}} \setminus \widetilde{\bf R}_{\{ \lhd \widetilde{\bf Z} \}}$.  But the denominator is independent of these variables, by assumption. Thus, it follows that fixing on a valid partial order with missing data and fixing on a total order consistent with this partial order, as in causal inference, yield equivalent kernels. 
	
	The conclusion follows by Lemma 55 in \cite{richardson17nested}.
	
	
	
\end{proof}

\begin{lema}{\ref{lem:property}}
	Consider a DAG ${\cal G}({\bf X}^{({\bf 1})}, {\bf R}, {\bf O}, {\bf X})$ such that for every $R_i \in {\bf R}$,
	$\{ R_j | X_j^{(1)} \in \pa_{\cal G}(R_i) \} \cap \an_{\cal G}(R_i) = \emptyset$.  Then for every $R_i \in {\bf R}$, a fixing schedule $\lhd$
	for $\{ \{ R_j \} | R_j \in {\cal G}_{{\bf R} \cap \de_{\cal G}(R_i)} \}$ given by the partial order induced by the ancestrality relation on
	${\cal G}_{{\bf R} \cap \de_{\cal G}(R_i)}$ is valid in ${\cal G}({\bf X}^{({\bf 1})}, {\bf R}, {\bf O}, {\bf X})$, by taking each $\bf X^{(1)}_{\widetilde{{\bf Z}}} =  \bigcup_{{\bf Z} \in \{ \unlhd \widetilde{\bf Z} \}} {\bf X}^{({\bf 1})}_{\bf Z}$, 
	for every $\widetilde{\bf Z} \in \{ \unlhd \{ R_i \} \}$. Thus the target law is identified.
\end{lema}

\begin{proof}
	
	In order to prove that the target law is identified, we demonstrate that conditions (i-v) in Proposition \ref{prop:r} are satisfied for each $R_i$.
	
	Conditions (i) and (ii) are trivially satisfied as we choose to fix ${\bf Z} \subseteq {\bf R}$, and we choose no equivalence relation, thus $\quotient{\bf Z}{\sim}$ consists of singleton sets of $R$s.
	Condition (iii) is also trivial as each ${\bf X^{(1)}_{\widetilde{Z}}}$ is a union of the corresponding sets ${\bf X^{(1)}_{\widetilde{Y}}}$, for ${\bf \widetilde{Y}}$ earlier in the partial order. In the proposed order we never fix elements in ${\bf X^{(1)}}$, and propose to keep elements in ${\bf X^{(1)}} \cap \pa_{\G}(R_j)$ for every $R_j \in {\bf Z}$. In particular, this also includes $R_i$, satisfying condition (iv).
	
	Finally, we show that the proposed schedule $\lhd$ is valid by showing that each ${\bf \widetilde{Z}} \in \quotient{{\bf Z}}{\sim}$ is fixable. There are 3 conditions for an element $\widetilde{\bf Z}$ to be fixable as mentioned in section \ref{sec:alg}. We go through each of these conditions and demonstrate each $\widetilde{\bf Z}$ in $\quotient{\bf Z}{\sim}$ is a valid fixing in $\phi_{\lhd_{\widetilde{\bf Z}}}({\cal G})$ where $\lhd$ is the proposed fixing schedule above.
	
	In the proposed schedule each $\widetilde{{\bf Z}}$ is a singleton $R_j \in \quotient{\bf Z}{\sim}$ that we are trying to fix in a graph $\phi_{\lhd_{R_j}}({\cal G})$. Since ${\bf X}^{({\bf 1})}_{R_j} = {\bf X}^{({\bf 1})}$, $\phi_{\lhd_{R_j}}({\cal G})$ is a CDAG. Thus, ${\cal D}(\phi_{\lhd_{R_j}}({\cal G}))$ is just sets of singleton vertices. In particular, ${\bf D}_{R_j} = \{ R_j \}$. Further, by definition of the schedule, it must be that $\de_{\phi_{\lhd_{R_j}}({\cal G})}(R_j)=\{R_j\}$. This satisfies condition (i).
	
	For condition (ii), we note that ${\bf S} \subseteq \nd_{\phi_{\lhd_{R_j}}({\cal G})}(R_j)$ else, ${\bf S}$ contains some $R_k \in \de_{\G}(R_j)$ which should have been fixed prior to $R_j$ by the proposed partial order. Thus, it follows that ${\bf S} \cap \{ R_j \}=\emptyset$.
	
	
	Finally, following the partial order, and under the assumption stated in the lemma, ${\bf R}_{\{R_j\}} \subseteq \{\lhd R_j \}$. 
	We have also proved that ${\bf S} \subseteq \nd_{\phi_{\lhd_{R_j}}({\cal G})}(R_j)$. Therefore, $R_j \ci ({\bf S} \cup {\bf R}_{\{R_j\}}) \setminus \mb_{\phi_{\lhd_{R_j}}({\cal G})}(R_j) | \mb_{\phi_{\lhd_{R_j}}({\cal G})}(R_j)$.
	
	Since each $\widetilde{{\bf Z}}$ is fixable, the proposed partial order $\lhd$ for each $R_i$ is valid. Therefore, all five conditions in Proposition \ref{prop:r} are satisfied concluding the target law is ID. 
	
\end{proof}

\subsection*{B. An example to illustrate the algorithm}

We walk the reader through identification of the target law for the missing data DAG shown in Fig.~\ref{fig:example}(a) in order to demonstrate the full generality of our missing ID algorithm. 
As a reminder, the target law is identified by (\ref{eqn:chain}) if we are able to identify $p(R_i|\pa_{\cal G}(R_i)) \vert_{{\bf R = 1}}$ for each $R_i \in {\bf R}$. The identification of these conditional densities are shown in equations (i) through (viii). For a clearer presentation of this example, we switch to one-column format.

\clearpage
\onecolumn

\begin{figure}
	\begin{center}
		\scalebox{0.56}{
			\begin{tikzpicture}[>=stealth, node distance=2cm]
			\tikzstyle{every node}=[font=\Large]
			\tikzstyle{format} = [thick, circle, minimum size=1.0mm, inner sep=3pt]
			\tikzstyle{square} = [draw, thick, minimum size=1.0mm, inner sep=3pt]
			\begin{scope}
			\path[->, very thick]
			node[format] (x11) {$X^{(1)}_1$}
			node[format, right of=x11] (x21) {$X^{(1)}_2$}
			node[format, right of=x21] (x31) {$X^{(1)}_3$}
			node[format, right of=x31] (x41) {$X^{(1)}_4$}
			node[format, below of=x11] (r1) {$R_1$}
			node[format, below of=x21] (r2) {$R_2$}
			node[format, below of=x31] (r3) {$R_3$}
			node[format, below of=x41] (r4) {$R_4$}
			node[format, left of=r1] (r5) {$R_5$}
			node[format, left of=r5] (r6) {$R_6$}
			node[format, left of=r6] (r7) {$R_7$}
			node[format, left of=r7] (r8) {$R_8$}
			node[format, above of=r5] (x51) {$X^{(1)}_5$}
			node[format, above of=r6] (x61) {$X^{(1)}_6$}
			node[format, above of=r7] (x71) {$X^{(1)}_7$}
			node[format, above of=r8] (x81) {$X^{(1)}_8$}
			(x11) edge[blue, bend left] (x31)
			(x11) edge[blue] (r2)
			(x31) edge[blue] (x21)
			(x21) edge[blue] (r1)
			(x11) edge[blue] (r4)
			(x41) edge[blue] (r1)
			(x41) edge[blue] (r3)
			(r2) edge[blue] (r1)
			(r2) edge[blue] (r3)
			(r3) edge[blue, bend left=15] (r1)
			(r4) edge[blue, bend left=15] (r2)
			(r1) edge[blue] (r5)
			(r1) edge[blue, bend left=15] (r6)
			(r8) edge[blue] (r7)
			(r8) edge[blue, bend right=15] (r6)
			(r4) edge[blue, bend left=25] (r8)
			(x51) edge[blue] (r1)
			(x51) edge[blue] (r6)
			(x61) edge[blue] (r1)
			(x61) edge[blue] (r5)
			(x61) edge[blue] (r7)
			(x61) edge[blue] (r8)
			(x71) edge[blue] (r8)
			(x71) edge[blue] (r6)
			(x81) edge[blue] (r1)
			(x81) edge[blue] (x71)
			(x71) edge[blue] (x61)
			(x61) edge[blue] (x51)
			(x71) edge[blue, bend left] (x51)
			node[below of=r1, yshift=-1cm] (node) {\LARGE (a) $\cal G$}
			;
			\end{scope}
			\begin{scope}[xshift=9cm, yshift=-2cm]
			\path[->, very thick]
			node[format] (r1) {$R_1$}
			node[format, below of=r1, xshift=-0.5cm] (r5) {$R_5$}
			node[format, below of=r1, xshift=0.5cm] (r6) {$R_6$}
			(r1) edge[blue, -] (r5)
			(r1) edge[blue, -] (r6)
			node[below of=r5, xshift=0.5cm, yshift=1cm] (node) {\LARGE (b)}
			;
			\end{scope}
			\begin{scope}[xshift=12cm, yshift=0cm]
			\path[->, very thick]
			node[format] (r2) {$R_2$}
			node[format, below of=r2, xshift=-0.5cm] (r1) {$R_1$}
			node[format, below of=r2, xshift=0.5cm] (r0) {}
			node[format, below of=r1, xshift=-0.5cm] (r5) {$R_5$}
			node[format, below of=r1, xshift=0.5cm] (r6) {$R_6$}
			node[format, below of=r0, xshift=0.5cm] (r3) {$R_3$}
			(r2) edge[blue, -] (r1)
			(r1) edge[blue, -] (r5)
			(r1) edge[blue, -] (r6)
			(r2) edge[blue, -] (r3)
			node[below of=r6, xshift=0cm, yshift=1cm] (node) {\LARGE(c)}
			;
			\end{scope}
			\begin{scope}[xshift=15cm, yshift=-2cm]
			\path[->, very thick]
			node[format] (r8) {$R_8$}
			node[format, below of=r8, xshift=-0.5cm] (r6) {$R_6$}
			node[format, below of=r8, xshift=0.5cm] (r7) {$R_7$}
			(r8) edge[blue, -] (r6)
			(r8) edge[blue, -] (r7)
			node[below of=r6, xshift=0.5cm, yshift=1cm] (node) {\LARGE (d)}
			;
			\end{scope}
			\begin{scope}[xshift=19cm, yshift=2cm]
			\path[->, very thick]
			node[format] (r4) {$R_4$}
			node[format, below of=r4, xshift=-0.5cm] (r2) {$R_2$}
			node[format, below of=r2, xshift=-0.5cm] (r13) {$\{R_1, R_3\}$}
			node[format, below of=r13, xshift=-0.5cm] (r5) {$R_5$}
			node[format, below of=r4, xshift=0.5cm] (r0) {}
			node[format, below of=r0, xshift=0.5cm] (r8) {$R_8$}
			node[format, below of=r8, xshift=0.5cm] (r7) {$R_7$}
			node[format, below of=r8, xshift=-0.75cm] (r6) {$R_6$}
			(r4) edge[blue, -] (r2)
			(r2) edge[blue, -] (r13)
			(r13) edge[blue, -] (r5)
			(r4) edge[blue, -] (r8)
			(r8) edge[blue, -] (r6)
			(r8) edge[blue, -] (r7)
			(r13) edge[blue, -] (r6)
			node[below of=r6, xshift=0cm, yshift=1cm] (node) {\LARGE (e)}
			;
			\end{scope}
			\end{tikzpicture}
		}
	\end{center}
	\caption{(a) A complex missing data DAG used to illustrate the general techniques used in our algorithm (b-e) The corresponding fixing schedules of $R$s.} 
	\label{fig:example}
\end{figure}
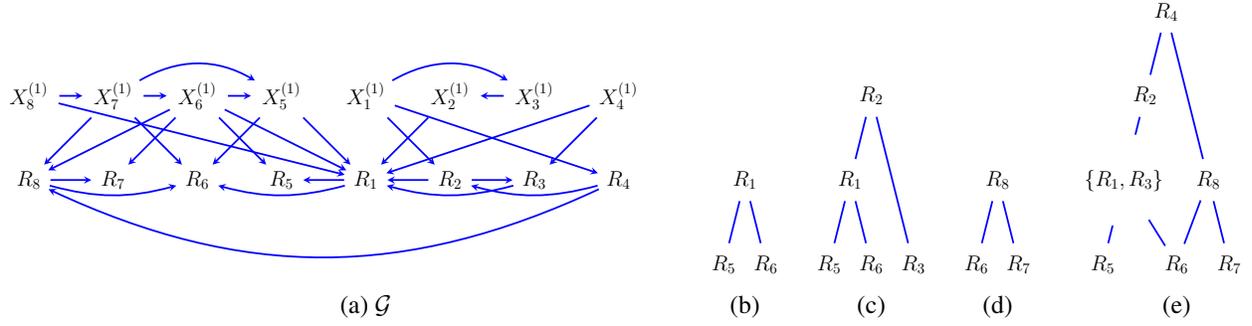

We start with $\{R_3, R_5, R_6, R_7\}$. The fixing schedules for these are empty and we obtain the following immediately from the original distribution.
\begin{enumerate}[label=(\roman*)]
	\item $p(R_3 | \pa(R_3)) = p(R_3 | R_2, X^{(1)}_4) = p(R_3 | R_2, X_4, \bias{R_4})$,  
	\item $p(R_5 | \pa(R_5)) = p(R_5 | R_1, X^{(1)}_6) = p(R_5 | R_1, X_6, \bias{R_6})$, 
	\item $p(R_6 | \pa(R_6)) = p(R_6 | R_1, R_8, X^{(1)}_5, X^{(1)}_7) = p(R_6 | R_1, R_8, X_5, X_7, \bias{R_5, R_7})$,  
	\item $p(R_7 | \pa(R_7)) = p(R_7 | R_8, X^{(1)}_6) = p(R_7 | R_8, X_6, \bias{R_6})$. 
\end{enumerate}

\begin{center}
	\textcolor{brown}{\rule{0.75\textwidth}{1pt}}
\end{center}

For $R_1$, we choose ${\bf Z}=\{R_1, R_5, R_6\}$, and no equivalence relations. Thus, $\quotient{{\bf Z}}{\sim}=\{ \{R_1\}, \{R_5\},\{R_6\}  \}$. The fixing schedule $\lhd$ is a partial order shown in Fig.~\ref{fig:example}(b) where $R_5$ and $R_6$ are incompatible, and $R_5 \prec R_1$, $R_6 \prec R_1$. 
Starting with the original $\cal G$ in Fig.~\ref{fig:example}(a), fixing $R_5$ and $R_6$ in parallel yields the following kernel. 
\begin{align}
q_{r_1}( {\bf X}\setminus \{X_5, X_6\}, X^{(1)}_5, X^{(1)}_6, {\bf R} \setminus \{R_5, R_6\} | \bias{R_5, R_6}) =  \frac{p({\bf X, R=1})}{p(R_5 | R_1, X^{(1)}_6) \ p(R_6 | R_1, R_8, X^{(1)}_5, X^{(1)}_7)\vert_{{\bf R = 1}}}, 
\label{eq:R1}
\end{align}
where the propensity scores in the denominator are identified using (ii) and (iii). The CADMG corresponding to this fixing operation is shown in Fig.~\ref{fig:R1R8}(a). 
\begin{align*}
\text{(v)} \quad p(R_1 | \pa(R_1))\vert_{{\bf R = 1}} 
&= p(R_1 | R_2, R_3, X^{(1)}_2, X^{(1)}_4, X^{(1)}_5, X^{(1)}_6)\vert_{{\bf R = 1}}  \\ 
&= q_{r_1} (R_1 | R_2, R_3, X^{(1)}_2, X^{(1)}_4, X_5, X_6,  \bias{R_5, R_6})\vert_{{\bf R = 1}} \\
&= q_{r_1} (R_1 | R_3, X_2, X^{(1)}_4, X_5, X_6,  \bias{R_2, R_5, R_6})\vert_{{\bf R = 1}}  \\
&= q_{r_1} (R_1 | R_3, X_2, X_4, X_5, X_6,  \bias{R_2, R_4, R_5, R_6})\vert_{{\bf R = 1}}  \quad  \text{(by d-sep)}\\
\end{align*}%
where the last term can be obtained using kernel operations (conditioning+marginalization) on $q_{r_1}(.|.)$ defined in \eqref{eq:R1}.

\begin{center}
	\textcolor{brown}{\rule{0.75\textwidth}{1pt}}
\end{center}

A similar procedure is applicable to $R_8$, where $\quotient{{\bf Z}}{\sim}=\{ \{R_8\}, \{R_7\},\{R_6\}  \}$; Fig.~\ref{fig:example}(d). Starting with the original $\cal G$ in Fig.~\ref{fig:example}(a), fixing $R_6$ and $R_7$ in parallel yields the following kernel. 
\begin{align}
q_{r_8}( {\bf X}\setminus \{X_6, X_7\}, X^{(1)}_6, X^{(1)}_7, {\bf R} \setminus \{R_6, R_7\} | \bias{R_6, R_7}) =  \frac{p({\bf X, R=1})}{p(R_6 | R_1, R_8, X^{(1)}_5, X^{(1)}_7) \ p(R_7 | R_8, X^{(1)}_6)\vert_{{\bf R = 1}} }, 
\label{eq:R8}
\end{align}%
where the propensity scores in the denominator are identified using (iii) and (iv). The CADMG corresponding to this fixing operation is shown in Fig.~\ref{fig:R1R8}(b). 
\begin{align*}
\text{(vi)} \quad p(R_8 | \pa(R_8))\vert_{{\bf R = 1}} 
&= p(R_8 | R_4, X^{(1)}_6, X^{(1)}_7)\vert_{{\bf R = 1}}  \\
&= q_{r_8}(R_8 | R_4, X^{(1)}_6, X^{(1)}_7, \bias{R_6, R_7})\vert_{{\bf R = 1}}  \\
&= q_{r_8}(R_8 | R_4, X_6, X_7, \bias{R_6, R_7})\vert_{{\bf R = 1}}  
\end{align*}%
where the last term can be obtained using kernel operations (conditioning+marginalization) on $q_{r_8}(.|.)$ defined in \eqref{eq:R8}. 

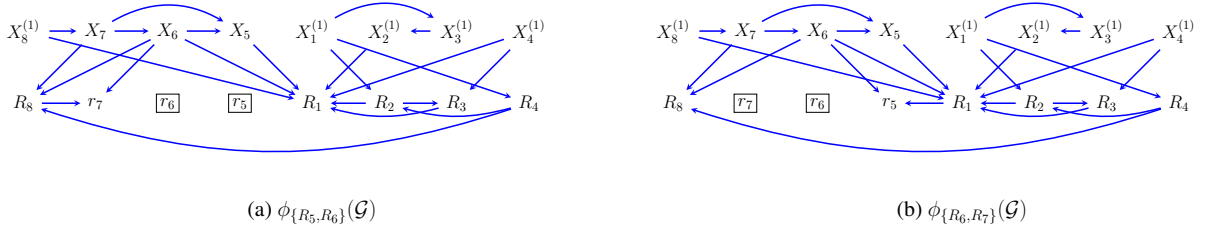
\begin{figure}
	\begin{center}
		\scalebox{0.48}{
			\begin{tikzpicture}[>=stealth, node distance=2cm]
			\tikzstyle{every node}=[font=\Large]
			\tikzstyle{format} = [thick, circle, minimum size=1.0mm, inner sep=3pt]
			\tikzstyle{square} = [draw, thick, minimum size=1.0mm, inner sep=3pt]
			\begin{scope}
			\path[->, very thick]
			node[format] (x11) {$X^{(1)}_1$}
			node[format, right of=x11] (x21) {$X^{(1)}_2$}
			node[format, right of=x21] (x31) {$X^{(1)}_3$}
			node[format, right of=x31] (x41) {$X^{(1)}_4$}
			node[format, below of=x11] (r1) {$R_1$}
			node[format, below of=x21] (r2) {$R_2$}
			node[format, below of=x31] (r3) {$R_3$}
			node[format, below of=x41] (r4) {$R_4$}
			node[square, left of=r1] (r5) {$r_5$}
			node[square, left of=r5] (r6) {$r_6$}
			node[format, left of=r6] (r7) {$r_7$}
			node[format, left of=r7] (r8) {$R_8$}
			node[format, above of=r5] (x51) {$X_5$}
			node[format, above of=r6] (x61) {$X_6$}
			node[format, above of=r7] (x71) {$X_7$}
			node[format, above of=r8] (x81) {$X^{(1)}_8$}
			(x11) edge[blue, bend left] (x31)
			(x11) edge[blue] (r2)
			(x31) edge[blue] (x21)
			(x21) edge[blue] (r1)
			(x11) edge[blue] (r4)
			(x41) edge[blue] (r1)
			(x41) edge[blue] (r3)
			(r2) edge[blue] (r1)
			(r2) edge[blue] (r3)
			(r3) edge[blue, bend left=15] (r1)
			(r4) edge[blue, bend left=15] (r2)
			(r8) edge[blue] (r7)
			(r4) edge[blue, bend left=18] (r8)
			(x51) edge[blue] (r1)
			(x61) edge[blue] (r1)
			(x61) edge[blue] (r7)
			(x61) edge[blue] (r8)
			(x71) edge[blue] (r8)
			(x81) edge[blue] (r1)
			(x81) edge[blue] (x71)
			(x71) edge[blue] (x61)
			(x61) edge[blue] (x51)
			(x71) edge[blue, bend left] (x51)
			node[below of=r1, yshift=-1cm] (node) {\LARGE (a) $\phi_{\{R_5, R_6\}}(\cal G)$}
			;
			\end{scope}
			\begin{scope}[xshift=18cm]
			\path[->, very thick]
			node[format] (x11) {$X^{(1)}_1$}
			node[format, right of=x11] (x21) {$X^{(1)}_2$}
			node[format, right of=x21] (x31) {$X^{(1)}_3$}
			node[format, right of=x31] (x41) {$X^{(1)}_4$}
			node[format, below of=x11] (r1) {$R_1$}
			node[format, below of=x21] (r2) {$R_2$}
			node[format, below of=x31] (r3) {$R_3$}
			node[format, below of=x41] (r4) {$R_4$}
			node[format, left of=r1] (r5) {$r_5$}
			node[square, left of=r5] (r6) {$r_6$}
			node[square, left of=r6] (r7) {$r_7$}
			node[format, left of=r7] (r8) {$R_8$}
			node[format, above of=r5] (x51) {$X_5$}
			node[format, above of=r6] (x61) {$X_6$}
			node[format, above of=r7] (x71) {$X_7$}
			node[format, above of=r8] (x81) {$X^{(1)}_8$}
			(x11) edge[blue, bend left] (x31)
			(x11) edge[blue] (r2)
			(x31) edge[blue] (x21)
			(x21) edge[blue] (r1)
			(x11) edge[blue] (r4)
			(x41) edge[blue] (r1)
			(x41) edge[blue] (r3)
			(r2) edge[blue] (r1)
			(r2) edge[blue] (r3)
			(r3) edge[blue, bend left=15] (r1)
			(r4) edge[blue, bend left=15] (r2)
			(r1) edge[blue] (r5)
			(r4) edge[blue, bend left=18] (r8)
			(x51) edge[blue] (r1)
			(x61) edge[blue] (r1)
			(x61) edge[blue] (r5)
			(x61) edge[blue] (r8)
			(x71) edge[blue] (r8)
			(x81) edge[blue] (r1)
			(x81) edge[blue] (x71)
			(x71) edge[blue] (x61)
			(x61) edge[blue] (x51)
			(x71) edge[blue, bend left] (x51)
			node[below of=r1, yshift=-1cm] (node) {\LARGE (b) $\phi_{\{R_6, R_7\}}(\cal G)$}
			;
			\end{scope}
			\end{tikzpicture}
		}
	\end{center}
	\caption{(a) Graph corresponding to the kernel obtained in \eqref{eq:R1} (b) Graph corresponding to the kernel obtained in \eqref{eq:R8}.} 
	\label{fig:R1R8}
\end{figure}

\begin{center}
	\textcolor{brown}{\rule{0.75\textwidth}{1pt}}
\end{center}

For $R_2$, we choose ${\bf Z}=\{ R_1, R_2, R_3, R_5, R_6\}$, and no equivalence relations. Thus, $\quotient{{\bf Z}}{\sim}=\{ \{R_1\}, \{R_2\},\{R_3\}, \{R_5\}, \{R_6\} \}$. The fixing schedule $\lhd$ is a partial order where $R_3, R_5, R_6$ are incompatible and $R_5, R_6 \prec R_1 \prec R_2$ and $R_3 \prec R_2$ as shown in Fig.~\ref{fig:example}(c). 
In addition, the portion of the fixing schedule involving $R_1$, $R_5$, and $R_6$ is executed in a latent projection ADMG where we treat $X^{(1)}_2$ as being hidden as shown in Fig.~\ref{fig:R2}(a), while the portion of the fixing schedule involving $R_3$ is executed in the original graph, Fig.~\ref{fig:example}(a).
\begin{align}
\text{(vii)} \quad p(R_2 | R_4, X^{(1)}_1)  = q_{r_2}( R_2 | R_4, X^{(1)}_1, \bias{R_1, R_3}), 
\end{align}
where $q_{r_2}$ corresponds to the kernel obtained by following the partial order of fixing $R_3$ and $R_1$, separately. That is,
\begin{align}
q_{r_2}( . |  \bias{R_1, R_3}) = 
\frac{p(\bf X, R=1)}
{q^{1}_{r_2}( R_1 | R_2, R_3, X_2, X_5, X_6, X^{(1)}_3, X^{(1)}_8, \bias{R_5, R_6}) \ p(R_3 | R_2, X^{(1)}_4)}.
\end{align}%
The propensity score for $R_3$ is obtained from (i) and $q^1_{r_2}$ is the kernel obtained by fixing $R_5$ and $R_6$ in parallel in a graph where $X^{(1)}_2$ is treated as hidden, as shown in Figures \ref{fig:R2}(a) and (b). That is, 
\begin{align*}
q^{1}_{r_2}({\bf X}\setminus \{X_5, X_6\}, X^{(1)}_5, X^{(1)}_6, {\bf R} \setminus \{R_5, R_6\} | \bias{R_5, R_6}) = \frac{p({\bf X, R=1})}{p(R_5 | R_1, X^{(1)}_6) \	p(R_6 | R_1, R_8, X^{(1)}_5, X^{(1)}_7)\vert_{{\bf R = 1}}}. 
\end{align*}%
The propensity scores in the denominator above are identified using (ii) and (iii). For clarity, the CADMGs corresponding to fixing $R_1$ and $R_3$ are illustrated in Figures \ref{fig:R2}(c) and (d). 

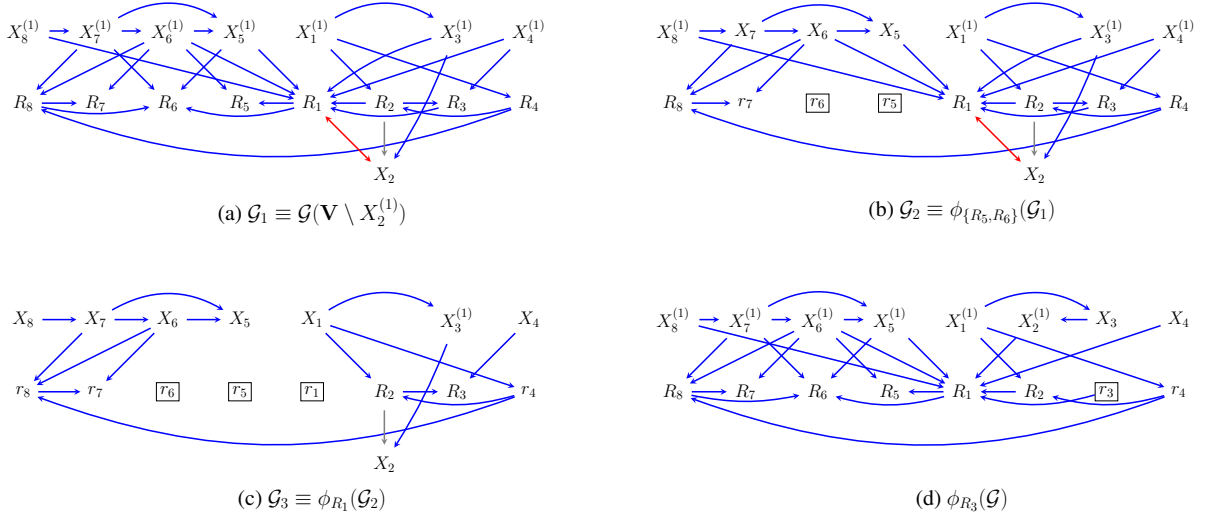
\begin{figure}
	\begin{center}
		\scalebox{0.48}{
			\begin{tikzpicture}[>=stealth, node distance=2cm]
			\tikzstyle{every node}=[font=\Large]
			\tikzstyle{format} = [thick, circle, minimum size=1.0mm, inner sep=3pt]
			\tikzstyle{square} = [draw, thick, minimum size=1.0mm, inner sep=3pt]
			\begin{scope}
			\path[->, very thick]
			node[format] (x11) {$X^{(1)}_1$}
			node[format, right of=x11] (x21) {}
			node[format, right of=x21] (x31) {$X^{(1)}_3$}
			node[format, right of=x31] (x41) {$X^{(1)}_4$}
			node[format, below of=x11] (r1) {$R_1$}
			node[format, below of=x21] (r2) {$R_2$}
			node[format, below of=r2] (x2) {$X_2$}
			node[format, below of=x31] (r3) {$R_3$}
			node[format, below of=x41] (r4) {$R_4$}
			node[format, left of=r1] (r5) {$R_5$}
			node[format, left of=r5] (r6) {$R_6$}
			node[format, left of=r6] (r7) {$R_7$}
			node[format, left of=r7] (r8) {$R_8$}
			node[format, above of=r5] (x51) {$X^{(1)}_5$}
			node[format, above of=r6] (x61) {$X^{(1)}_6$}
			node[format, above of=r7] (x71) {$X^{(1)}_7$}
			node[format, above of=r8] (x81) {$X^{(1)}_8$}
			(x11) edge[blue, bend left] (x31)
			(x11) edge[blue] (r2)
			(x31) edge[blue, bend right=10] (r1)
			(x11) edge[blue] (r4)
			(x41) edge[blue] (r1)
			(x41) edge[blue] (r3)
			(r2) edge[blue] (r1)
			(r2) edge[blue] (r3)
			(r3) edge[blue, bend left=15] (r1)
			(r4) edge[blue, bend left=15] (r2)
			(r1) edge[blue] (r5)
			(r1) edge[blue, bend left=15] (r6)
			(r8) edge[blue] (r7)
			(r8) edge[blue, bend right=12] (r6)
			(r4) edge[blue, bend left=20] (r8)
			(x51) edge[blue] (r1)
			(x51) edge[blue] (r6)
			(x61) edge[blue] (r1)
			(x61) edge[blue] (r5)
			(x61) edge[blue] (r7)
			(x61) edge[blue] (r8)
			(x71) edge[blue] (r8)
			(x71) edge[blue] (r6)
			(x81) edge[blue] (r1)
			(x81) edge[blue] (x71)
			(x71) edge[blue] (x61)
			(x61) edge[blue] (x51)
			(r1) edge[red, <->] (x2)
			(x31) edge[blue, bend left=5] (x2)
			(r2) edge[gray] (x2)
			(x71) edge[blue, bend left] (x51)
			node[below of=r1, yshift=-1cm] (node) {\LARGE (a) ${\cal G}_1 \equiv {\cal G}({\bf V}\setminus X^{(1)}_2)$}
			;
			\end{scope}
			\begin{scope}[xshift=18cm]
			\path[->, very thick]
			node[format] (x11) {$X^{(1)}_1$}
			node[format, right of=x11] (x21) {}
			node[format, right of=x21] (x31) {$X^{(1)}_3$}
			node[format, right of=x31] (x41) {$X^{(1)}_4$}
			node[format, below of=x11] (r1) {$R_1$}
			node[format, below of=x21] (r2) {$R_2$}
			node[format, below of=r2] (x2) {$X_2$}
			node[format, below of=x31] (r3) {$R_3$}
			node[format, below of=x41] (r4) {$R_4$}
			node[square, left of=r1] (r5) {$r_5$}
			node[square, left of=r5] (r6) {$r_6$}
			node[format, left of=r6] (r7) {$r_7$}
			node[format, left of=r7] (r8) {$R_8$}
			node[format, above of=r5] (x51) {$X_5$}
			node[format, above of=r6] (x61) {$X_6$}
			node[format, above of=r7] (x71) {$X_7$}
			node[format, above of=r8] (x81) {$X^{(1)}_8$}
			(x11) edge[blue, bend left] (x31)
			(x11) edge[blue] (r2)
			(x31) edge[blue, bend right=10] (r1)
			(x11) edge[blue] (r4)
			(x41) edge[blue] (r1)
			(x41) edge[blue] (r3)
			(r2) edge[blue] (r1)
			(r2) edge[blue] (r3)
			(r3) edge[blue, bend left=15] (r1)
			(r4) edge[blue, bend left=15] (r2)
			(r8) edge[blue] (r7)
			(r4) edge[blue, bend left=20] (r8)
			(x51) edge[blue] (r1)
			(x61) edge[blue] (r1)
			(x61) edge[blue] (r7)
			(x61) edge[blue] (r8)
			(x71) edge[blue] (r8)
			(x81) edge[blue] (r1)
			(x81) edge[blue] (x71)
			(x71) edge[blue] (x61)
			(x61) edge[blue] (x51)
			(r1) edge[red, <->] (x2)
			(x31) edge[blue, bend left=5] (x2)
			(r2) edge[gray] (x2)
			(x71) edge[blue, bend left] (x51)
			node[below of=r1, yshift=-1cm] (node) {\LARGE (b) ${\cal G}_2 \equiv \phi_{\{R_5, R_6\}} ({\cal G}_1)$}
			;
			\end{scope}
			\begin{scope}[yshift=-8cm]
			\path[->, very thick]
			node[format] (x11) {$X_1$}
			node[format, right of=x11] (x21) {}
			node[format, right of=x21] (x31) {$X^{(1)}_3$}
			node[format, right of=x31] (x41) {$X_4$}
			node[square, below of=x11] (r1) {$r_1$}
			node[format, below of=x21] (r2) {$R_2$}
			node[format, below of=r2] (x2) {$X_2$}
			node[format, below of=x31] (r3) {$R_3$}
			node[format, below of=x41] (r4) {$r_4$}
			node[square, left of=r1] (r5) {$r_5$}
			node[square, left of=r5] (r6) {$r_6$}
			node[format, left of=r6] (r7) {$r_7$}
			node[format, left of=r7] (r8) {$r_8$}
			node[format, above of=r5] (x51) {$X_5$}
			node[format, above of=r6] (x61) {$X_6$}
			node[format, above of=r7] (x71) {$X_7$}
			node[format, above of=r8] (x81) {$X_8$}
			(x11) edge[blue, bend left] (x31)
			(x11) edge[blue] (r2)
			(x11) edge[blue] (r4)
			(x41) edge[blue] (r3)
			(r2) edge[blue] (r3)
			(r4) edge[blue, bend left=15] (r2)
			(r8) edge[blue] (r7)
			(r4) edge[blue, bend left=20] (r8)
			(x61) edge[blue] (r7)
			(x61) edge[blue] (r8)
			(x71) edge[blue] (r8)
			(x81) edge[blue] (x71)
			(x71) edge[blue] (x61)
			(x61) edge[blue] (x51)
			(x31) edge[blue, bend left=5] (x2)
			(r2) edge[gray] (x2)
			(x71) edge[blue, bend left] (x51)
			node[below of=r1, yshift=-1cm] (node) {\LARGE (c) ${\cal G}_3 \equiv \phi_{R_1} ({\cal G}_2)$}
			;
			\end{scope}
			\begin{scope}[xshift=18cm, yshift=-8cm]
			\path[->, very thick]
			node[format] (x11) {$X^{(1)}_1$}
			node[format, right of=x11] (x21) {$X^{(1)}_2$}
			node[format, right of=x21] (x31) {$X_3$}
			node[format, right of=x31] (x41) {$X_4$}
			node[format, below of=x11] (r1) {$R_1$}
			node[format, below of=x21] (r2) {$R_2$}
			node[square, below of=x31] (r3) {$r_3$}
			node[format, below of=x41] (r4) {$r_4$}
			node[format, left of=r1] (r5) {$R_5$}
			node[format, left of=r5] (r6) {$R_6$}
			node[format, left of=r6] (r7) {$R_7$}
			node[format, left of=r7] (r8) {$R_8$}
			node[format, above of=r5] (x51) {$X^{(1)}_5$}
			node[format, above of=r6] (x61) {$X^{(1)}_6$}
			node[format, above of=r7] (x71) {$X^{(1)}_7$}
			node[format, above of=r8] (x81) {$X^{(1)}_8$}
			(x11) edge[blue, bend left] (x31)
			(x11) edge[blue] (r2)
			(x31) edge[blue] (x21)
			(x21) edge[blue] (r1)
			(x11) edge[blue] (r4)
			(x41) edge[blue] (r1)
			(r2) edge[blue] (r1)
			(r3) edge[blue, bend left=15] (r1)
			(r4) edge[blue, bend left=15] (r2)
			(r1) edge[blue] (r5)
			(r1) edge[blue, bend left=15] (r6)
			(r8) edge[blue] (r7)
			(r8) edge[blue, bend right=11] (r6)
			(r4) edge[blue, bend left=20] (r8)
			(x51) edge[blue] (r1)
			(x51) edge[blue] (r6)
			(x61) edge[blue] (r1)
			(x61) edge[blue] (r5)
			(x61) edge[blue] (r7)
			(x61) edge[blue] (r8)
			(x71) edge[blue] (r8)
			(x71) edge[blue] (r6)
			(x81) edge[blue] (r1)
			(x81) edge[blue] (x71)
			(x71) edge[blue] (x61)
			(x61) edge[blue] (x51)
			(x71) edge[blue, bend left] (x51)
			node[below of=r1, yshift=-1cm] (node) {\LARGE (d) $\phi_{R_3}(\cal G)$}
			;
			\end{scope}
			\end{tikzpicture}
		}
	\end{center}
	\caption{Execution of the fixing schedule to obtain the propensity score for $R_1$ (a) Latent projection ADMG obtained by projecting out $X^{(1)}_2$ (b) Fixing $R_5$ and $R_6$ in ${\cal G}_1$ (c) Fixing $R_1$ in ${\cal G}_2$ (d) Fixing $R_3$ in the original graph.} 
	\label{fig:R2}
\end{figure}

\begin{center}
	\textcolor{brown}{\rule{0.75\textwidth}{1pt}}
\end{center}

Finally, for $R_4$, we choose ${\bf Z}=\{ {\bf R } \}$ and equivalence relation $R_1 \sim R_3$. Thus, $\quotient{{\bf Z}}{\sim}=\{ \{R_1, R_3\}, \{R_2\}, \{R_4\}, \{R_5\}, \{R_6\}, \{R_7\}, \{R_8\} \}$. The fixing schedule $\lhd$ is a partial order where $R_5, R_6 \prec \{ R_1, R_3\} \prec R_2 \prec R_4$ and $R_6, R_7 \prec R_8 \prec R_4$ as shown in Fig.~\ref{fig:example}(e). In addition, the portion of the fixing schedule involving $R_5$, $R_6$, $\{R_1, R_3\}$, and $R_2$ is executed in a latent projection ADMG where we treat $X^{(1)}_2$ and $X^{(1)}_4$ as hidden variables, shown in Fig.~\ref{fig:R4}(b), while the portion of the fixing schedule involving $R_6, R_7$, and $R_8$ is executed in the original graph, Fig.~\ref{fig:example}(a).
\begin{align}
\text{(viii)} \quad p(R_4 | X^{(1)}_1)  = q_{r_4}( R_4 | X^{(1)}_1, \bias{R_2, R_8}), 
\end{align}
where $q_{r_4}$ corresponds to the kernel obtained by following the partial order of fixing $R_2$ and $R_8$, separately. That is,
\begin{align}
q_{r_4}( . |  \bias{R_2, R_8}) = 
\frac{p(\bf X, R=1)}
{q^{1}_{r_4}( R_2 | R_4, X_2) \ q^{2}_{r_4} (R_8 | R_4, X_6, X_7)}.
\end{align}%
$q^1_{r_4}$ is the kernel obtained by fixing the set $\{R1, R3\}$ in graph $\G_2$ shown in Fig.~\ref{fig:R4}(c). That is,

\begin{align*}
q^{1}_{r_4}(. | \bias{R_1, R_3, R_5, R_6}) 
&= \frac{q^{3}_{r_4}(. | \bias{R_5, R_6})}{q^{3}_{r_4}(R_ 1, R_3 | R_2, R_4, X_2, X^{(1)}_3, X_4)} \\
&= \frac{q^{3}_{r_4}(. | \bias{R_5, R_6})}
{q^{3}_{r_4}(R_ 1 | R_2, R_4, X_2, X_3, X_4, \bias{R_3}) \  q^{3}_{r_4}(R_ 3 | R_2, R_4, X_2, X_4)}
\end{align*} %
$q^3_{r_4}$ is the kernel obtained by fixing $R_5$ and $R_6$ in parallel in the graph $\G_1$ shown in Fig.~\ref{fig:R4}(b). That is,

\begin{align*}
q^{3}_{r_4}(. | \bias{R_5, R_6}) = \frac{p({\bf X, R=1})}{p(R_5 | R_1, X^{(1)}_6) \	p(R_6 | R_1, R_8, X^{(1)}_5, X^{(1)}_7)\vert_{{\bf R = 1}}}.
\end{align*} %
The propensity scores in the denominator above are identified using (ii) and (iii).

Finally, $q^{2}_{r_4}$ is the kernel obtained by fixing $R_6$ and $R_7$ in parallel in the original graph $\G$, shown in Fig.~\ref{fig:example}(a).  That is,

\begin{align*}
q^{2}_{r_4}(. | \bias{R_6, R_7}) = \frac{p({\bf X, R=1})}{p(R_6 | R_1, R_8, X^{(1)}_5, X^{(1)}_7  )  \ p(R_7 | R_8, X^{(1)}_6)\vert_{{\bf R = 1}}}.
\end{align*}%
The propensity scores in the denominator above are identified using (iii) and (iv). For clarity, the CADMG corresponding to fixing $R_8$ is illustrated in Figures \ref{fig:R4}(a). 

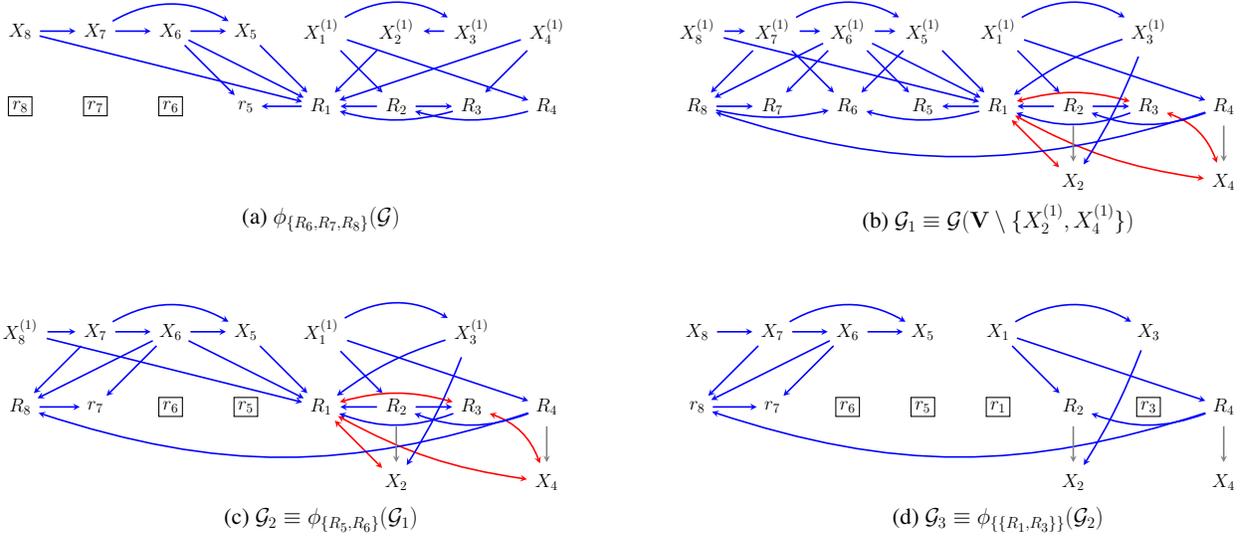
\begin{figure}
	\begin{center}
		\scalebox{0.5}{
			\begin{tikzpicture}[>=stealth, node distance=2cm]
			\tikzstyle{every node}=[font=\Large]
			\tikzstyle{format} = [thick, circle, minimum size=1.0mm, inner sep=3pt]
			\tikzstyle{square} = [draw, thick, minimum size=1.0mm, inner sep=3pt]
			\begin{scope}
			\path[->, very thick]
			node[format] (x11) {$X^{(1)}_1$}
			node[format, right of=x11] (x21) {$X^{(1)}_2$}
			node[format, right of=x21] (x31) {$X^{(1)}_3$}
			node[format, right of=x31] (x41) {$X^{(1)}_4$}
			node[format, below of=x11] (r1) {$R_1$}
			node[format, below of=x21] (r2) {$R_2$}
			node[format, below of=x31] (r3) {$R_3$}
			node[format, below of=x41] (r4) {$R_4$}
			node[format, left of=r1] (r5) {$r_5$}
			node[square, left of=r5] (r6) {$r_6$}
			node[square, left of=r6] (r7) {$r_7$}
			node[square, left of=r7] (r8) {$r_8$}
			node[format, above of=r5] (x51) {$X_5$}
			node[format, above of=r6] (x61) {$X_6$}
			node[format, above of=r7] (x71) {$X_7$}
			node[format, above of=r8] (x81) {$X_8$}
			(x11) edge[blue, bend left] (x31)
			(x11) edge[blue] (r2)
			(x31) edge[blue] (x21)
			(x21) edge[blue] (r1)
			(x11) edge[blue] (r4)
			(x41) edge[blue] (r1)
			(x41) edge[blue] (r3)
			(r2) edge[blue] (r1)
			(r2) edge[blue] (r3)
			(r3) edge[blue, bend left=15] (r1)
			(r4) edge[blue, bend left=15] (r2)
			(r1) edge[blue] (r5)
			(x51) edge[blue] (r1)
			(x61) edge[blue] (r1)
			(x61) edge[blue] (r5)
			(x81) edge[blue] (r1)
			(x81) edge[blue] (x71)
			(x71) edge[blue] (x61)
			(x61) edge[blue] (x51)
			(x71) edge[blue, bend left] (x51)
			node[below of=r1, yshift=-1cm] (node) {\LARGE (a) $\phi_{\{R_6, R_7, R_8\}} (\cal G)$}
			;
			\end{scope}
			\begin{scope}[xshift=18cm]
			\path[->, very thick]
			node[format] (x11) {$X^{(1)}_1$}
			node[format, right of=x11] (x21) {}
			node[format, right of=x21] (x31) {$X^{(1)}_3$}
			node[format, right of=x31] (x41) {}
			node[format, below of=x11] (r1) {$R_1$}
			node[format, below of=x21] (r2) {$R_2$}
			node[format, below of=r2] (x2) {$X_2$}
			node[format, below of=x31] (r3) {$R_3$}
			node[format, below of=x41] (r4) {$R_4$}
			node[format, below of=r4] (x4) {$X_4$}
			node[format, left of=r1] (r5) {$R_5$}
			node[format, left of=r5] (r6) {$R_6$}
			node[format, left of=r6] (r7) {$R_7$}
			node[format, left of=r7] (r8) {$R_8$}
			node[format, above of=r5] (x51) {$X^{(1)}_5$}
			node[format, above of=r6] (x61) {$X^{(1)}_6$}
			node[format, above of=r7] (x71) {$X^{(1)}_7$}
			node[format, above of=r8] (x81) {$X^{(1)}_8$}
			(x11) edge[blue, bend left] (x31)
			(x11) edge[blue] (r2)
			(x31) edge[blue, bend right=9] (r1)
			(x11) edge[blue] (r4)
			(r2) edge[blue] (r1)
			(r1) edge[red, <->] (x2)
			(r1) edge[red, <->, bend left=15] (r3)
			(r3) edge[red, <->, bend left=25] (x4)
			(r1) edge[red, <->, bend right=10] (x4)
			(x31) edge[blue, bend left=5] (x2)
			(r2) edge[blue] (r3)
			(r3) edge[blue, bend left=20] (r1)
			(r4) edge[blue, bend left=20] (r2)
			(r1) edge[blue] (r5)
			(r1) edge[blue, bend left=15] (r6)
			(r8) edge[blue] (r7)
			(r8) edge[blue, bend right=12] (r6)
			(r4) edge[blue, bend left=18] (r8)
			(x51) edge[blue] (r1)
			(x51) edge[blue] (r6)
			(x61) edge[blue] (r1)
			(x61) edge[blue] (r5)
			(x61) edge[blue] (r7)
			(x61) edge[blue] (r8)
			(x71) edge[blue] (r8)
			(x71) edge[blue] (r6)
			(x81) edge[blue] (r1)
			(x81) edge[blue] (x71)
			(x71) edge[blue] (x61)
			(x61) edge[blue] (x51)
			(r2) edge[gray] (x2)
			(r4) edge[gray] (x4)
			(x71) edge[blue, bend left] (x51)
			node[below of=r1, yshift=-1cm] (node) {\LARGE (b) ${\cal G}_1 \equiv {\cal G}({\bf V}\setminus \{X^{(1)}_2, X^{(1)}_4\})$ }
			;
			\end{scope}
			\begin{scope}[yshift=-8cm]
			\path[->, very thick]
			node[format] (x11) {$X^{(1)}_1$}
			node[format, right of=x11] (x21) {}
			node[format, right of=x21] (x31) {$X^{(1)}_3$}
			node[format, right of=x31] (x41) {}
			node[format, below of=x11] (r1) {$R_1$}
			node[format, below of=x21] (r2) {$R_2$}
			node[format, below of=r2] (x2) {$X_2$}
			node[format, below of=x31] (r3) {$R_3$}
			node[format, below of=x41] (r4) {$R_4$}
			node[format, below of=r4] (x4) {$X_4$}
			node[square, left of=r1] (r5) {$r_5$}
			node[square, left of=r5] (r6) {$r_6$}
			node[format, left of=r6] (r7) {$r_7$}
			node[format, left of=r7] (r8) {$R_8$}
			node[format, above of=r5] (x51) {$X_5$}
			node[format, above of=r6] (x61) {$X_6$}
			node[format, above of=r7] (x71) {$X_7$}
			node[format, above of=r8] (x81) {$X^{(1)}_8$}
			(x11) edge[blue, bend left] (x31)
			(x11) edge[blue] (r2)
			(x31) edge[blue, bend right=9] (r1)
			(x11) edge[blue] (r4)
			(r2) edge[blue] (r1)
			(r1) edge[red, <->] (x2)
			(r1) edge[red, <->, bend left=15] (r3)
			(r3) edge[red, <->, bend left=25] (x4)
			(r1) edge[red, <->, bend right=10] (x4)
			(x31) edge[blue, bend left=5] (x2)
			(r2) edge[blue] (r3)
			(r3) edge[blue, bend left=20] (r1)
			(r4) edge[blue, bend left=20] (r2)
			(r8) edge[blue] (r7)
			(r4) edge[blue, bend left=18] (r8)
			(x51) edge[blue] (r1)
			(x61) edge[blue] (r1)
			(x61) edge[blue] (r7)
			(x61) edge[blue] (r8)
			(x71) edge[blue] (r8)
			(x81) edge[blue] (r1)
			(x81) edge[blue] (x71)
			(x71) edge[blue] (x61)
			(x61) edge[blue] (x51)
			(r2) edge[gray] (x2)
			(r4) edge[gray] (x4)
			(x71) edge[blue, bend left] (x51)
			node[below of=r1, yshift=-1cm] (node) {\LARGE (c) $\G_2 \equiv \phi_{\{R_5, R_6\}}(\G_1)$}
			;
			\end{scope}
			\begin{scope}[yshift=-8cm, xshift=18cm]
			\path[->, very thick]
			node[format] (x11) {$X_1$}
			node[format, right of=x11] (x21) {}
			node[format, right of=x21] (x31) {$X_3$}
			node[format, right of=x31] (x41) {}
			node[square, below of=x11] (r1) {$r_1$}
			node[format, below of=x21] (r2) {$R_2$}
			node[format, below of=r2] (x2) {$X_2$}
			node[square, below of=x31] (r3) {$r_3$}
			node[format, below of=x41] (r4) {$R_4$}
			node[format, below of=r4] (x4) {$X_4$}
			node[square, left of=r1] (r5) {$r_5$}
			node[square, left of=r5] (r6) {$r_6$}
			node[format, left of=r6] (r7) {$r_7$}
			node[format, left of=r7] (r8) {$r_8$}
			node[format, above of=r5] (x51) {$X_5$}
			node[format, above of=r6] (x61) {$X_6$}
			node[format, above of=r7] (x71) {$X_7$}
			node[format, above of=r8] (x81) {$X_8$}
			(x11) edge[blue, bend left] (x31)
			(x11) edge[blue] (r2)
			(x11) edge[blue] (r4)
			(x31) edge[blue, bend left=5] (x2)
			(r4) edge[blue, bend left=20] (r2)
			(r8) edge[blue] (r7)
			(r4) edge[blue, bend left=18] (r8)
			(x61) edge[blue] (r7)
			(x61) edge[blue] (r8)
			(x71) edge[blue] (r8)
			(x81) edge[blue] (x71)
			(x71) edge[blue] (x61)
			(x61) edge[blue] (x51)
			(r2) edge[gray] (x2)
			(r4) edge[gray] (x4)
			(x71) edge[blue, bend left] (x51)
			node[below of=r1, yshift=-1cm] (node) {\LARGE (d) $\G_3 \equiv \phi_{\{\{R_1, R_3\}\}}(\G_2)$}
			;
			\end{scope}
			\end{tikzpicture}
		}
	\end{center}
	\caption{Execution of the fixing schedule to obtain the propensity score for $R_4$ (a) CADMG obtained by following the schedule to get the propensity score for $R_8$ (b) Latent projection ADMG obtained by projecting out $X^{(1)}_2$ and $X^{(1)}_4$ (c) Fixing $R_5$ and $R_6$ in ${\cal G}_1$ (d) Fixing $R_1$ in ${\cal G}_2$. } 
	\label{fig:R4}
\end{figure}



\subsection*{C. Table for Lemma \ref{lem:colluder}}
\begin{minipage}{0.2\textwidth}
	\begin{center}
		\scalebox{0.75}{
			\begin{tikzpicture}[>=stealth, node distance=1.3cm]
			\tikzstyle{format} = [thick, circle, minimum size=1.0mm,
			inner sep=0pt]
			\begin{scope}
			\path[->, very thick]
			node[format] (x11) {$X^{(1)}_1$}
			node[format, right of=x11] (x21) {$X^{(1)}_2$}
			node[format, below of=x11] (r1) {$R_1$}		
			node[format, below of=x21] (r2) {$R_2$}
			node[format, below of=r1] (x1) {$X_1$}
			node[format, below of=r2] (x2) {$X_2$}
			(x11) edge[blue] (r2)
			(r1) edge[blue] (r2)
			(r1) edge[gray] (x1)
			(r2) edge[gray] (x2)
			(x11) edge[gray, bend right] (x1)
			(x21) edge[gray, bend left] (x2)
			;
			\end{scope} 
			\end{tikzpicture}
		}
	\end{center}
\end{minipage}
\begin{minipage}{0.7\textwidth}
	\scalebox{0.8}{
		\begin{tabular}{| c | c |}
			\hline
			$R_1$ & $p(R_1)$ \\ \hline
			$0$  & $a$     \\ 
			$1$  & $1-a$  \\ \hline
		\end{tabular}
		
		\hspace{0.5cm}
		
		\begin{tabular}{| c | c |}
			\hline
			$X^{(1)}_1$ & $p(X^{(1)}_1)$ \\ \hline
			$0$  & $b$     \\ 
			$1$  & $1-b$ \\ \hline
		\end{tabular}
		
		\hspace{0.5cm}
		
		\begin{tabular}{| c | c |}
			\hline
			$X^{(1)}_2$ & $p(X^{(1)}_2)$ \\ \hline
			$0$  & $c$     \\ 
			$1$  & $1 - c$  \\ \hline
		\end{tabular}
		
		\hspace{0.5cm}
		
		\begin{tabular}{| c : c c | c |}
			\hline
			$R_2$ & $R_1$ & $X^{(1)}_1$ & $p(R_2 | R_1, X^{(1)}_1)$   \\ \hline
			$0$  & $0$ & $0$  & $\red{d}$     \\
			$1$  & $0$ & $0$ & $\red{1 - d}$ \\ \hline 
			$0$  & $1$ & $0$ & $e$     \\
			$1$  & $1$ & $0$ & $1- e$  \\
			$0$  & $0$ & $1$  & $\red{f}$   \\
			$1$  & $0$ & $1$ & $\red{1 - f}$ \\ \hline 
			$0$  & $1$ & $1$ & $g$     \\
			$1$  & $1$ & $1$ & $1- g$  \\ \hline
		\end{tabular}
	}
\end{minipage}

\begin{table}[h]
	\begin{center}
		\scalebox{0.9}{
			\begin{tabular}{ | c | c | c | c | c | c | c | c | }
				\hline
				$R_1$  & $R_2$   & $X^{(1)}_1$   &  $X^{(1)}_2$   & p(Full Law)  & $X_1$   & $X_2$   &  p(Observed Law)    \\ \hline
				\multirow{4}{*}{0} & \multirow{4}{*}{0} & 0     & 0    &  $abc\red{d}$   & \multirow{4}{*}{?} & \multirow{4}{*}{?} & \multirow{4}{*}{$a\red{\Big[ db + f(1-b)) \Big]}$}   \\ 
				&   &  1 & 0  & $a(1-b)c\red{f}$   &   &  &  \\
				&   &  0 & 1  & $ab(1-c)\red{d}$   &   &  &  \\
				&   &  1 & 1  & $a(1-b)(1-c)\red{f}$   &   &  &  \\  
				
				\hline \hline 
				
				\multirow{4}{*}{1} & \multirow{4}{*}{0} & 0     & 0    &  $(1-a)ebc$   &  \multirow{2}{*}{$0$} & \multirow{4}{*}{?} & \multirow{2}{*}{$(1-a)eb$}   \\ 
				&   &  1 & 0  & $(1-a)g(1-b)c$   &  &  &  \\
				&   &  0 & 1  & $(1-a)eb(1-c)$   &  \multirow{2}{*}{$1$}  &  &  \multirow{2}{*}{$(1-a)g(1-b)$}  \\
				&   &  1 & 1  & $(1-a)g(1-b)(1-c)$   &  &  &  \\ 
				
				\hline \hline 
				
				\multirow{4}{*}{0} & \multirow{4}{*}{1} & 0     & 0    &  $abc\red{(1-d)}$   & \multirow{4}{*}{?} & \multirow{2}{*}{$0$}   & \multirow{2}{*}{$ac\Big[1 - \red{\Big( db + f(1-b)  \Big)} \Big]$}   \\ 
				&   &  1 & 0  & $a(1-b)c\red{(1-f)}$   &  &  &  \\
				&   &  0 & 1  & $ab(1-c)\red{(1-d)}$   &  & \multirow{2}{*}{$1$}  &  \multirow{2}{*}{$a(1-c)\Big[1 - \red{\Big( db + f(1-b)  \Big)}\Big]$}  \\
				&   &  1 & 1  & $a(1-b)(1-c)\red{(1-f)}$   &  &  &  \\ 
				
				\hline \hline 
				
				\multirow{4}{*}{1} & \multirow{4}{*}{1} & 0     & 0    &  $(1-a)(1-e)bc$   & $0$ & $0$   &  $(1-a)(1-e)bc$    \\     
				&  & $1$   & $0$    &  $(1-a)(1-g)(1-b)c$   & $1$   & $0$  &  $(1-a)(1-g)(1-b)c$   \\      
				&  & $0$   & $1$    &  $(1-a)(1-e)b(1-c)$   & $0$   & $0$  &  $(1-a)(1-e)b(1-c)$   \\ 
				&  & $1$   & $1$    &  $(1-a)(1-g)(1-b)(1-c)$   & $1$   & $1$  &  $(1-a)(1-g)(1-b)(1-c)$   \\  \hline
			\end{tabular}
		}
	\end{center}
\end{table}

\vspace{0.2cm}
Any pair of $\{d, f\}$ would lead to different full laws. However, as long as $ db + f(1-b)$ stays constant, the observe law would agree across all different full laws (which include infinitely many models). This is a general characterization of non-identifiable models with two binary random variables.

\end{document}